\documentclass[journal,twoside,web]{ieeecolor}
\usepackage{generic} 

\usepackage{import} 

\usepackage{amsmath,amssymb,amsthm,mathtools,bm}
\usepackage[cmyk,rgb]{xcolor}
\usepackage{graphicx}
\usepackage{cite}
\usepackage{url}
\usepackage{microtype}
\usepackage{comment}
\usepackage{verbatim}
\usepackage{adjustbox}
\usepackage{grffile} 
\usepackage{pgfplots}
\pgfplotsset{compat=newest}
\usepackage{tikz}
\usetikzlibrary{
  arrows.meta, calc, intersections, automata, positioning,
  shapes, decorations.pathmorphing, decorations.shapes,
  decorations.text, angles, quotes, patterns, shadings
}
\usepackage{booktabs}

\usepackage[ruled,vlined,linesnumbered]{algorithm2e}

\usepackage[font=footnotesize,labelfont=bf]{caption}
\makeatletter
\long\def\@makecaption#1#2{%
  \vskip\abovecaptionskip
  \footnotesize \textbf{#1.} #2\par
  \vskip\belowcaptionskip}
\makeatother

\usepackage[hidelinks]{hyperref}
\hypersetup{
    colorlinks=true,
    linkcolor=blue,
    citecolor=red,
}

\tikzset{
  pil/.style={
    ->,
    thick,
    shorten <=2pt,
    shorten >=2pt
  }
}

\theoremstyle{plain} 
\newtheorem{theorem}{Theorem}
\newtheorem{lemma}{Lemma}
\newtheorem{prop}{Proposition}
\newtheorem{coroll}{Corollary}

\theoremstyle{definition}
\newtheorem{defn}{Definition}

\newtheorem{algorthm}{Algorithm}

\newtheorem*{problem*}{Problem}

\theoremstyle{remark}
\newtheorem{remarkx}{Remark}
\newenvironment{remark}
  {\noindent \pushQED{\qed} \remarkx}
  {\popQED\endremarkx}

\newtheorem{examplex}{Example}
\newenvironment{example}
  {\noindent \pushQED{\qed} \examplex}
  {\popQED\endexamplex}

\newcommand{\scalemath}[2]{\scalebox{#1}{\mbox{\ensuremath{\displaystyle #2}}}}

\definecolor{applegreen}{rgb}{0.0, 0.7, 0.0}

\useRomanappendicesfalse

\title{\LARGE \bf Leaderless Collective Motion in Affine Formation Control over the Complex Plane}

\author{Jesus Bautista, \IEEEmembership{Student, IEEE}, Enric Morella, \IEEEmembership{Student, IEEE}, Lili Wang, \IEEEmembership{Member, IEEE}, \\ and Hector Garcia de Marina, \IEEEmembership{Member, IEEE} %
    \thanks{This work is supported by the RYC2020-030090-I grant from the Spanish Ministry of Science, the ERC Starting Grant \emph{iSwarm} 101076091, and the National Natural Science Foundation of China under Grant No. 62403233. J. Bautista, E. Morella and H.G. de Marina are with the Department of Computer Engineering, Automation, and Robotics, and with IMAG, University of Granada, Spain. {\tt\small \{jesusbv, emorella, hgdemarina\}@ugr.es}. L. Wang is  with Shenzhen Key Laboratory of Control Theory and Intelligent Systems, School of Automation and Intelligent Manufacturing, Southern University of Science and Technology, Shenzhen, China, 518055. {\tt\small wangll@sustech.edu.cn}.} %
}

\begin{document}

\maketitle
\thispagestyle{empty}
\pagestyle{empty}

\normalsize

\begin{abstract}
We propose a method for the collective maneuvering of affine formations in the plane by modifying the original weights of the Laplacian matrix used to achieve static formations of robot swarms. Specifically, the resulting collective motion is characterized as a time-varying affine transformation of a reference configuration, or \emph{shape}. Unlike the traditional leader-follower strategy, our leaderless scheme allows agents to maintain distinct and possibly time-varying velocities, enabling a broader range of collective motions, including all the linear combinations of translations, rotations, scaling and shearing of a reference shape. Our analysis provides the analytic solution governing the resulting collective motion, explicitly designing the eigenvectors and eigenvalues that define this motion as a function of the modified weights in the new Laplacian matrix. To facilitate a more tractable analysis and design of affine formations in 2D, we propose the use of complex numbers to represent all relevant information. Simulations with up to 20 agents validate the theoretical results.
\end{abstract}

\begin{IEEEkeywords}
    Distributed formation control, Multi-agent systems, Complex Laplacian, Autonomous systems.
\end{IEEEkeywords}


\section{Introduction} 
\label{sec: intro}

World-class roboticists and industry experts firmly indicate that swarm technology, such as distributed controllers, is key to achieve reconfigurable robotic tasks requiring high resilience and limitless scalability, particularly in vast, unstructured, and dynamic environments \cite{challenges}. These tasks include specific missions in search \& rescue \cite{SandR}, exploration \cite{burgard2005coordinated}, or environmental monitoring \cite{dunbabin2012robots}. However, the full potential of robot swarms will be realized only if we can engineer predictable \textit{collective behavior} that emerges primarily from \textit{local robot interactions}, ideally without requiring explicit inter-robot communication.

Among the most studied forms of collective behavior is the ability of multi-agent teams to form and maintain precise geometric patterns. The literature offers various formation control strategies, with those grounded in consensus theory attracting significant attention due to their scalability, robustness, and low communication requirements \cite{surveyCons, lin2016complexlap, lin2016necessary}. Of particular relevance is the approach in \cite{lin2016necessary}, which leverages a \textit{generalized graph Laplacian} to ensure asymptotic convergence to formations invariant under \textit{affine transformations}. This invariance enables broad collective behaviors while preserving formation geometry. While \cite{zhao2018affine} extended this to leader-follower architectures and subsequent works addressed directed graphs \cite{yang2020}, higher-order dynamics \cite{affineTripleInt2019, affineHighOrder2020}, and disturbance robustness \cite{affineLayered2021}, relying on designated leaders introduces structural limitations and reduced robustness against key agents' failure.

In practical scenarios with inconsistent measurements, existing formation control algorithms exhibit robustness issues, including undesired stationary motions and distorted formations \cite{Sun2018, ug_problems, inconsistent_mes}. Rather than treating these imperfections as disturbances, we propose understanding their effects to design \textit{emergent collective behavior} that maneuvers multi-agent teams \cite{chen2021maneuvering, de2016distributed}. Our leaderless, distributed formation maneuvering approach deliberately injects \textit{structured imperfections}, such as artificial scaling or biased displacement measurements, into the Laplacian-based controller from \cite{lin2016necessary}. These imperfections become \textit{motion parameters} embedded in the original Laplacian weights, requiring no additional control layer while preserving global stability guarantees—an advantage over leader-follower approaches \cite{zhao2018affine, yang2020, Fang2024}.

Initially explored in \cite{hgdemarina_complex} for complex formations and extended in our conference version \cite{hgdemarina_affine} with real-valued biases, this journal version introduces three novel contributions: (1) we present the key idea of using complex number for 2D affine formation control, enhancing techniques from \cite{hgdemarina_complex,lin2016complexlap} to analyze \emph{affine motions}; (2) we show how complex representation enables basis-selection strategy, characterizing any desired collective motion as linear combinations of fundamental affine modes, significantly simplifying weight computation; (3) we derive complete explicit analytical solutions to closed-loop dynamics, typically no available in leader-follower frameworks.

All results in this work consider \textit{single-integrator dynamics}, offering analytical tractability while capturing core formation behavior \cite{formationcontrol}. Importantly, this choice does not restrict applicability—single-integrator models serve as high-level controllers in established practical implementations such as guiding vector fields \cite{gvf_classic, gvf_parametric, PF_survey}, where convergence and stability guarantees are preserved for complex robotic platforms.

The rest of this paper is organized as follows. Section \ref{sec: pre} introduces necessary notation and graph theory concepts. Section \ref{sec: R2toC} explains planar affine formation codification in $\mathbb{C}$ and designs real Laplacian-based formation control. Section \ref{sec: Lmod} introduces the modified Laplacian matrix, analyzes stability, and provides analytical solutions. Section \ref{sec: sims} presents numerical simulations, and Section \ref{sec: conclusions} concludes with future work.


\section{Preliminaries}
\label{sec: pre}

\subsection{Notation and graph theory}

We consider $n\in\mathbb{N}$ mobile agents with $\iota$ denoting the complex unit. We use $|x|$ for the Euclidean norm of $x\in\mathbb{C}^n$, $|\mathcal{X}|$ for set cardinality, $\mathbf{1}_n \in \mathbb{C}^n$ for the $n$-dimensional vector of ones, and define $\overline A := A \otimes I_m \in \mathbb{R}^{pm\times mq}$ for matrix $A \in \mathbb{R}^{p\times q}$ using Kronecker product $\otimes$ and identity matrix $I_m$.

Given an affine space $\mathbb{A} := (\mathcal{A},V,F)$, where $\mathcal{A}$ is the points set and $V$ is a vector space over the field $F$, a \textit{reference frame} of $\mathbb{A}$ is defined as $\mathcal{O}_r := (p_r, \mathcal{B}_r)$, where $p_r\in \mathcal{A}$ is the origin, and $\mathcal{B}_r$ is a \textit{basis} of $V$. Throughout this paper, we will consider $\mathbb{C}^n$ as both the points set $\mathcal{A}$ and the vector space $V$ over the field $F = \mathbb{C}$.
The coordinate change from frame $\mathcal{O}_r$ to $\mathcal{O}_{r^\prime}$ is the affine map $T: \mathbb{C}^n \rightarrow \mathbb{C}^n$ defined by $x_r = T(x_r^\prime) = \mathcal{L}(x_r^\prime) + (p_r - p_r^\prime)$, where $x_r$ and $x_r^\prime$ represent the vector $x$ observed from $\mathcal{O}_r$ and $\mathcal{O}_{r^\prime}$, respectively, and $\mathcal{L}: \mathbb{C}^n \rightarrow \mathbb{C}^n$ is the linear basis change map. We denote $[x_r]_{\mathcal{B}}$ as the coordinates of $x_r$ with respect to the basis $\mathcal{B}$.

A \textit{graph} $\mathcal{G} = (\mathcal{V},\mathcal{E})$ consists of the node set $\mathcal{V} = \{1,2,\dots,n\}$ and the edge set $\mathcal{E} \subseteq (\mathcal{V} \times \mathcal{V})$. We consider \textit{undirected} graphs where edge $(i,j)\in \mathcal{E}$ implies $(j,i)\in\mathcal{E}$. The neighbors of node $i$ are $\mathcal{N}_i := {j\in\mathcal{V} : (i,j)\in\mathcal{E}}$. Let $w_{ij} \in \mathbb{R}$ be a weight associated with the edge $(i,j)$, then the \textit{Laplacian} matrix $L \in \mathbb{R}^{n\times n}$ of $\mathcal{G}$ is defined as
\begin{equation}
    \scalemath{0.9}{
    l_{ij} :=
    \begin{cases}
        \begin{array}{ll}
            \sum_{k\in\mathcal{N}_i}w_{ik} & \text{if} \quad i = j\\
            -\omega_{ij} & \text{if} \quad i \neq j \, \wedge \, j \in \mathcal{N}_i\\
            0 & \text{if} \quad i \neq j \, \wedge \, j \notin \mathcal{N}_i,
        \end{array}
    \end{cases}
    }
\end{equation}
satisfying $L\mathbf{1}_n = 0$. For undirected graphs, we choose one of the two arbitrary directions for each pair of neighboring nodes to construct the ordered edge set $\mathcal{Z}$ with elements $\mathcal{Z}_k = (\mathcal{Z}_k^{\text{head}}, \mathcal{Z}_k^{\text{tail}})$, $k \in \{1,\dots,\frac{|\mathcal{E}|}{2}\}$. From such an ordered set, we construct the \textit{incidence matrix} $B \in \mathbb{R}^{|\mathcal{V}|\times|\mathcal{Z}|}$ satisfying $B^\top \mathbf{1}_n = 0$ as
\begin{equation} \label{eq: B}
    \scalemath{0.9}{
    b_{ik} :=
    \begin{cases}
        \begin{array}{ll}
            +1 & \text{if} \quad i = \mathcal{Z}_k^{\text{tail}}\\
            -1 & \text{if} \quad i = \mathcal{Z}_k^{\text{head}}\\
            0 & \text{otherwise}.
        \end{array}
    \end{cases}
    }
\end{equation}


\section{Affine formation control in $\mathbb{R}^2$ over the complex plane}
\label{sec: R2toC}

\subsection{Framework and desired shape}

We encode the 2D position of each agent $i$ as $p_i \in \mathbb{C} \simeq \mathbb{R}^2$, where the real component accounts for the $x$-axis coordinate and the imaginary one for the $y$-axis coordinate. We stack all the positions $p_i$ in a single vector $p \in \mathbb{C}^n$ given with respect to the \textit{Earth-fixed} reference frame $\mathcal{O}_g = (p_g, \mathcal{B}_g)$, and we call it \textit{configuration}. Indeed, we can decode $p$ to a stack of vectors in $\mathbb{R}^{2n}$ by considering the linear map $\mathcal{R}: \mathbb{C}^n \rightarrow \mathbb{R}^{2n}$ that we define as $\mathcal{R}(p) = \operatorname{Re}(p) \otimes [1\;0]^\top + \operatorname{Im}(p) \otimes [0\;1]^\top$.

We define the \textit{framework} $\mathcal{F}$ as $(\mathcal{G},p)$, where we assign each agent's position $p_i$ to the node $i \in \mathcal{V}$, and the graph $\mathcal{G}$ establishes the set of neighbors $\mathcal{N}_i$ for each agent $i$.

We choose a \textit{reference shape} $p^* \in \mathbb{C}^n$ satisfying the following conditions: (1) $p^*$ is \textit{non-degenerate}; (2) the framework $(\mathcal{G}, p^*)$ is \textit{generic}, i.e., the coordinates do not satisfy any non-trivial algebraic equation with rational coefficients \cite{rigidity}; (3) $p^*$ is defined with respect to a \textit{body-fixed} reference frame $\mathcal{O}_b = (p_b, \mathcal{B}_b)$, where $p_b$ is at the centroid of $p^*$, as shown in \autoref{fig: p_star}; and (4) certain rigidity conditions on the framework $(\mathcal{G}, p^*)$, such as global or universal rigidity \cite{rigidity}, will be required to simplify some results. Although our results apply to arbitrary configurations, we retain point (1) since degenerate shapes (e.g., collinear agents) cannot generate motions such as rotation or shear.

\begin{figure}
    \centering
    \includegraphics[trim={0cm 0cm 0cm 0cm}, clip, width=1\columnwidth]{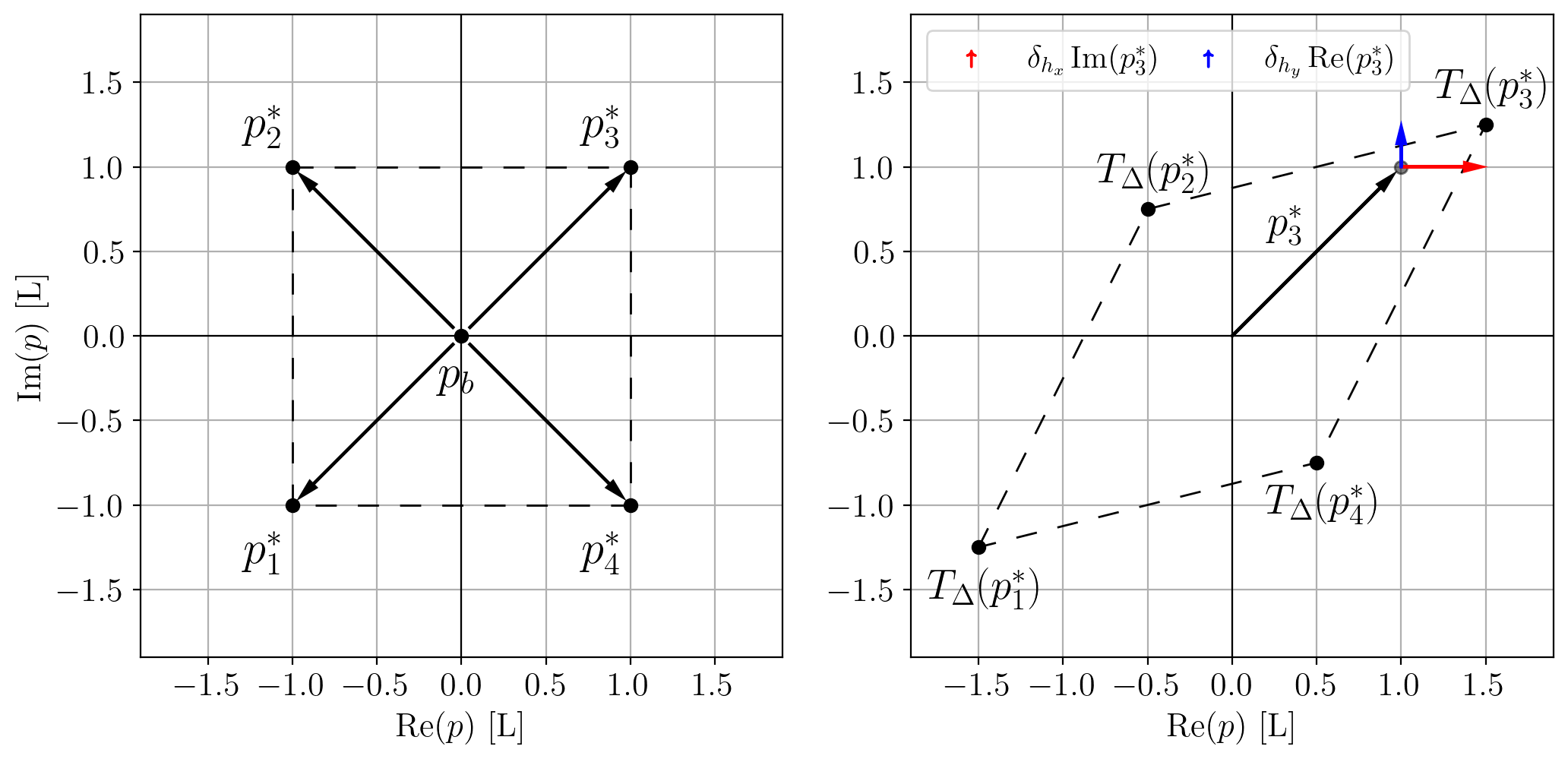}
    \caption{The left plot shows a square reference shape $p^*$ for a team of $n=4$ robots. In the right plot, this reference shape undergoes an affine transformation given by the linear transformation $T_\Delta$, with $\Delta = [\delta_x\; \delta_y\; \delta_{a_x}\; \delta_{a_y}\; \delta_{h_x}\; \delta_{h_y}] = [0 \; 0 \; 1 \; 1\; 0.5 \; 0.25]$, so that $T_\Delta(p^*) = (1 + \iota 0.25) \operatorname{Re}(p^*) + (0.5 + \iota) \operatorname{Im}(p^*)$.}
     \label{fig: p_star}
\end{figure}

An affine transformation of the reference shape $p^*$ in the plane, represented in $\mathbb{R}^2$, is given by $(I_n \otimes A) \mathcal{R}(p^*) + (\mathbf{1}_n \otimes b)$, where $A \in \mathbb{R}^{2\times 2}$ and $b \in \mathbb{R}^2$. Encoding the same affine transformation over $\mathbb{C}^n$ yields
\begin{equation} \label{eq: p_hat}
    T_\Delta(p^*) = c_1 \mathbf{1}_n + 
    c_2 \operatorname{Re}(p^*) +
    c_3 \operatorname{Im}(p^*),
\end{equation}
where $c_{\{1,2,3\}} \in \mathbb{C}$. Here, $T_\Delta: \mathbb{C}^n \rightarrow \mathbb{C}^n$ denotes the \emph{affine map} parameterized by the vector of \textit{affine coordinates} $\Delta = [\delta_x\;\; \delta_y\;\; \delta_{a_x}\;\; \delta_{a_y}\;\; \delta_{h_x}\;\; \delta_{h_y}]^\top \in \mathbb{R}^6$. The complex coefficients are defined as 
\begin{align} \label{eq: cs}
    c_1 = \delta_x + \iota \delta_y, \quad c_2 = \delta_{a_x} + \iota \delta_{h_y}, \quad c_3 = \delta_{h_x} + \iota \delta_{a_y},
\end{align}
so that component of $\Delta$ has a geometric interpretation: $\delta_x$ and $\delta_y$ account for horizontal and vertical translation, $\delta_{a_x}$ and $\delta_{a_y}$ control scaling parallel to the $x$-axis and $y$-axis of $\mathcal{O}_b$, respectively, and $\delta_{h_x}$ and $\delta_{h_y}$ produce shearing along these same axes. This codification is illustrated in \autoref{fig: p_star}.

Indeed, $T_\Delta(p^*)$ is the reference shape $p^*$ observed from a new reference frame $\mathcal{O}_\Delta = (p_\Delta, \mathcal{B}_\Delta)$, so the affine transformation $T_\Delta$ provides the change of coordinates from $\mathcal{O}_b$ to $\mathcal{O}_\Delta$, with $c_1 = p_\Delta - p_b$, and $c_{\{2,3\}}$ representing the change of basis. Therefore, we define the concept of \textit{desired shape} constructed from the reference shape $p^*$ as follows.
\begin{defn} [Desired shape]
    The \textit{desired shape set}, constructed from the reference shape $p^*$, is defined as
    \begin{align} \label{eq: S}
         \mathcal{S} := \{p = T_\Delta(p^*) \; | \; \Delta \in \mathbb{R}^6\}.
    \end{align}
    We say the formation is at the \textit{desired shape} when $p \in \mathcal{S}$.
\end{defn}

It is important to note that $\mathcal{S}$ does not contain all possible affine transformations of a given vector within $\mathbb{C}^n$, but rather all planar affine transformations of the reference shape $p^*$. Furthermore, $\mathcal{S}$ is a vector subspace of $\mathbb{C}^n$, spanned by $\mathbf{1}_n$, $\operatorname{Re}(p^*)$ and $\operatorname{Im}(p^*)$, which are ensured to be linearly independent by design. Specifically, we define the basis for the desired shape subspace as 
$\mathcal{B}_\mathcal{S} = \{\mathbf{1}_n, \; \operatorname{Re}(p^*), \; \operatorname{Im}(p^*)\},$
so that $[T_\Delta(p^*)]_{\mathcal{B}_\mathcal{S}} = [c_1\, c_2\, c_3]$. Finally, let us introduce the following straightforward result, since it will be convenient throughout this paper.
\begin{lemma} \label{lem: addition}
	Given a set $\mathcal{T}(x^*) = \{x = T_\Delta(x^*) \; | \; \Delta \in \mathbb{R}^6\}$, with $x^* \in \mathbb{C}^n$ and $T_\Delta$ as in \eqref{eq: p_hat}, any affine transformation in the plane of an element in $\mathcal{T}$ yields another element in $\mathcal{T}$.
\end{lemma}
\begin{proof}
	For the affine transformation, \eqref{eq: p_hat} and \eqref{eq: cs} lead us to
    \begin{align} \label{eq: T_delta_T_delta}
        \scalemath{0.76}{
        \begin{array}{rll}
            & T_{\Delta^\prime}(T_\Delta(p^*)) = \left[ (\delta_x^\prime + \delta_{a_x}^\prime \delta_x + \delta_y \delta_{h_x}^\prime) + \iota (\delta_y^\prime + \delta_{a_y}^\prime \delta_y + \delta_x \delta_{h_y}^\prime) \right] \mathbf{1}_n\\
            & \qquad + \left[ (\delta_{a_x}^\prime \delta_{a_x} + \delta_{h_x}^\prime \delta_{h_y}) + \iota (\delta_{h_y}^\prime \delta_{a_x} + \delta_{a_y}^\prime \delta_{h_y}) \right] \operatorname{Re}(p^*)\\
            & \qquad + \left[ (\delta_{a_x}^\prime \delta_{h_x} + \delta_{h_x}^\prime \delta_{a_y}) + \iota (\delta_{h_y}^\prime\delta_{h_x} + \delta_{a_y}^\prime \delta_{a_y}) \right] \operatorname{Im}(p^*),\\
        \end{array}}
    \end{align}
    which is a vector within $\mathcal{T}$, i.e., $T_{\Delta^\prime}(T_\Delta(p^*)) \in \mathcal{T}$.
\end{proof}



\subsection{Desired collective motion}

Considering a framework $\mathcal{F}=(\mathcal{G},p)$, we denote the time variation of every position $p_i$ as $v_i \in \mathbb{C}$. We stack all these velocity vectors $v_i$ in $v \in \mathbb{C}^n$, given with respect to $\mathcal{O}_g$, which represents the time variation of the configuration, and we call it \textit{collective motion}. Hence, as we do for the position, we consider a \textit{reference collective motion} $v^* \in \mathbb{C}^n$ for the team of agents, given with respect to $\mathcal{O}_b$. Although this $v^*$ can be arbitrary, in this paper it will be interesting to design it as $v^* = k_u T_\Delta(p^*)$. Let us note that, in contrast with other works \cite{zhao2018affine, Fang2024} where all robots share a common reference velocity, e.g., in \emph{leader-follower} schemes, here all the reference velocities $v_i^*$ in $v^*$ are different in general, enabling richer collective motions without the need to track any leader signal. Besides, we denote the reference collective motion $v^*$ observed from $\mathcal{O}_\Delta$ as $T_\Delta(v^*)$. In this case, the term $c_1 = \delta_x + \iota \delta_y$, which in $T_\Delta(p^*)$ represents a spatial translation, here represents the relative velocity between $\mathcal{O}_b$ and $\mathcal{O}_\Delta$, e.g., $c_1 = 0$ if both reference frames are moving at the same velocity. Finally, note that having $p = T_\Delta(p^*)$ not necessarily implies that $v = T_\Delta(v^*)$ with the same $\Delta$. Next, we will define the concept of \textit{desired collective motion} constructed from the reference collective motion $v^*$, and we will show how to design it so that the set $\mathcal{S}$ (\ref{eq: S}) remains invariant. 
\begin{defn} [Desired collective motion]
    The \textit{collective motion  set}, constructed from the reference collective motion $v^*$, is defined as
    \begin{align} \label{eq: M}
         \mathcal{M} := \{v = T_\Delta(v^*) \; | \; \Delta \in \mathbb{R}^6\}.
    \end{align}
    We say the formation exhibits the \textit{desired collective motion} when $v \in \mathcal{M}$.
\end{defn}

\begin{lemma} \label{lem: invariance}
	Consider $p(t_0) \in \mathcal{S}, t_0\in\mathbb{R}^+$. If $v(t) \in \mathcal{M}$ with $v^* = k_u T_{\Delta}(p^*)$, then $\mathcal{S}$ is invariant for all $t \geq t_0$.
\end{lemma}
\begin{proof}
	Firstly, we note that if two configurations $p_1$ and $p_2$ belong to $\mathcal{S}$ then it is straightforward to see that $(p_1 + p_2) \in \mathcal{S}$ by looking at (\ref{eq: p_hat}); in fact, we have that $T_{\Delta_1}(p^*) + T_{\Delta_2}(p^*) = T_{(\Delta_1 + \Delta_2)}(p^*)$. Secondly, when $p(t) \in \mathcal{S}$, we have that $v(t) = T_{\Delta_2(t)}(v^*) = k_u T_{\Delta_2(t)}(T_{\Delta_1}(p^*)) = k_u T_{\Delta_3(t)}(p^*)$ due to Lemma \ref{lem: addition}; thus, $p(t) = p(t_0) + \int_{t_0}^t k_u T_{\Delta_3(t)}(p^*) \mathrm{dt}$. Hence, since both $p(t_0)$ and $k_u T_{\Delta_3(t)}(p^*) \mathrm{dt}$ are within $\mathcal{S}$, and we take the integral in the sense of Riemann, then $p(t) \in \mathcal{S}$ for all $t \geq t_0$, i.e., $\mathcal{S}$ is invariant.
\end{proof}

\begin{remark}
The sets $\mathcal{M}$ in (\ref{eq: M}) and $\mathcal{S}$ in (\ref{eq: S}) might look the same; in fact, they are isomorphic. However, we should not forget that $\mathcal{M}$ and $\mathcal{S}$ contain velocities and positions, respectively. Nonetheless, for the sake of conciseness, we will omit the unit conversion $k_u$ for the rest of the paper.
\end{remark}

\begin{remark}
Note that a degenerated $p^*$ implies that $v^*$ is also \emph{degenerated}, i.e., a degenerated $p^*$ can restrict the family of possible motions $v^*$.
\end{remark}


\subsection{Stabilization of an affine static shape} \label{sec: static_shape}

We model the agents as point-mass particles, so that we can command their velocities as
\begin{equation} \label{eq: dyni}
    \dot p_i = u_i, \quad i\in\mathcal{V},
\end{equation}
where $u_i\in\mathbb{C}$ is the control input to the agent $i$. Stacking the agents' positions and inputs yields the compact form
\begin{equation} \label{eq: dyn}
    \dot p = u,
\end{equation}
where $u \in \mathbb{C}^n$ is the stacked vector of control actions. 

Since we aim for a distributed implementation, each $u_i$ must depend solely on local information; specifically, the relative positions $z_{ij} := (p_i - p_j), \, j \in \mathcal{N}_i$. In particular, the original algorithm that steers $p(t) \rightarrow \mathcal{S}$ with a static eventual configuration takes the form \cite{lin2016necessary}
\begin{equation} \label{eq: u_control}
    u_i = -\sum_{j \in \mathcal{N}_i} w_{ij} (p_i - p_j) = - \sum_{j \in \mathcal{N}_i} w_{ij} z_{ij},
\end{equation}
where $w_{ij} =  w_{ji} \in \mathbb{R}$ are weights whose design, along with the graph $\mathcal{G}$, will be explained shortly to ensure that the Laplacian matrix $L$ is positive semi-definite \cite{lin2016necessary, zhao2018affine}. The compact form of \eqref{eq: u_control} and \eqref{eq: dyni} in $\mathbb{R}^2$ is $\mathcal{R}(\dot p) = - \overline L \mathcal{R}(p)$, so $\mathcal{R}(p(t)) \rightarrow \operatorname{Ker}\{\overline L\}$ as $t \rightarrow \infty$. However, this expression can be simplified by encoding it in $\mathbb{C}$ as
\begin{equation} \label{eq: L_control}
    \dot p = -Lp,
\end{equation}
so that $p(t) \rightarrow \operatorname{Ker}\{L\}$. In the following technical result, we will show that the kernel of $L$ is the set $\mathcal{S}$ when the framework $\mathcal{F}$ is \textit{globally rigid} \cite{lin2016necessary} and we force the weights, besides the trivial solution, to satisfy the following balance 
\begin{equation} \label{eq: w_condition}
     \sum_{j\in\mathcal{N}_i} w_{ij}(p_i^* - p_j^*) = 0, \quad \forall i \in \mathcal{V}.
\end{equation}

\begin{lemma} \label{lem: L_eigen}
    Consider a reference shape $p^*$ so that $\mathbf{1}_n,\operatorname{Re}(p^*),\operatorname{Im}(p^*)\in\mathbb{R}^n$ are linearly independent, a globally rigid framework $\mathcal{F}$, and a Laplacian matrix $L$ whose weights are designed according to \eqref{eq: w_condition}, then
    \begin{align} \label{eq: ker_L}
    \operatorname{Ker}\{L\} = \{p = c_1 \mathbf{1}_n + c_2\operatorname{Re}(p^*) &+ c_3\operatorname{Im}(p^*) \; | \; \nonumber \\ & \quad c_1,c_2,c_3 \in \mathbb{C}\}.
    \end{align} 
\end{lemma}
\begin{proof}
    Firstly, note that $L\mathbf{1}_n = 0$ by the definition of Laplacian matrix. 
    Additionally, since $L$ is designed according to \eqref{eq: w_condition}, we have that
    \begin{align*}
        &\scalemath{0.9}{
        \sum_{j\in\mathcal{N}_i} w_{ij} (p_i^* - p_j^*) = \sum_{j\in\mathcal{N}_i} w_{ij}\left[\operatorname{Re}(p_i^* - p_j^*) + \iota\operatorname{Im}(p_i^* - p_j^*)\right] =} \\
        &\scalemath{0.9}{
        \sum_{j\in\mathcal{N}_i} w_{ij}\operatorname{Re}(p_i^* - p_j^*) + \iota\sum_{j\in\mathcal{N}_i} w_{ij}\operatorname{Im}(p_i^* - p_j^*) = 0, \quad \forall i \in \mathcal{V}},
   \end{align*}
   which in compact form gives $L\operatorname{Re}(p^*) + \iota L\operatorname{Im}(p^*) = 0$. By the fundamental properties of complex numbers, this implies $L \operatorname{Re}(p^*) = 0$ and $L \operatorname{Im}(p^*) = 0$. Furthermore, since $\mathcal{F}$ is globally rigid, we have that the kernel of $L$ has dimension $3$ \cite[Theorem 2.1]{lin2016necessary}. Consequently, since $\operatorname{Re}(p^*)$, $\operatorname{Im}(p^*)$, and $\mathbf{1}_n$ are linearly independent, they span the kernel of $L$.
\end{proof}

\begin{remark} \label{remark}
	Let $\lambda \in \mathbb{C}$ be an arbitrary eigenvalue of $L$, and $x\in\mathbb{R}^n$ its associated eigenvector. Then, it can be proved that $\lambda$ is also an eigenvalue of $\overline L$ with corresponding eigenvectors $x \otimes [1 \; 0]^\top$ and $x \otimes [0 \; 1]^\top$ \cite[Theorem 4.2.12]{Horn_Johnson_1991}. Therefore, the result in Lemma \ref{lem: L_eigen} can be directly applied to $L$ to be used in $\mathbb{R}^2$ (or higher), i.e., 
    \begin{align*}
        \operatorname{Ker}\{\overline L\} = \{p = &\mathbf{1}_{n} \otimes [r_1\;r_2]^\top + \operatorname{Re}(p^*) \otimes [r_3\;r_4]^\top \\ &+ \operatorname{Im}(p^*) \otimes [r_5\;r_6]^\top \; | \; r_1,\dots,r_6 \in \mathbb{R}\}.
    \end{align*} 
\end{remark}

According to Lemma \ref{lem: L_eigen}, the Laplacian matrix, with its weights satisfying all the constraints in \eqref{eq: w_condition}, has three zero eigenvalues, with eigenvectors $\mathbf{1}_n$, $\operatorname{Re}(p^*)$ and $\operatorname{Im}(p^*)$. 
Therefore, considering \eqref{eq: p_hat} and \eqref{eq: S}, we have that $\operatorname{Ker}\{L\} = \mathcal{S}$. Consequently, if the set of weights satisfying the constraints in \eqref{eq: w_condition} are designed so that the rest of the eigenvalues of $L$ are all positive, then $p(t)\rightarrow\mathcal{S}$ as $t \rightarrow \infty$ in (\ref{eq: L_control}).

When the framework $(\mathcal{G}, p^*)$ is \emph{generically and universally rigid} \footnote{Given a framework $(\mathcal{G}, p^*)$ with $p^*$ being generic, we say it is generically and universally rigid if, for any other configuration $(\mathcal{G}, q)$ with $q\in\mathbb{C}^s, \,s\in\mathbb{N}$, satisfying $||p_i^* - p_j^*|| = ||q_i - q_j||$, for all $(i,j) \in\mathcal{E}$, it also follows that $||p_i^* - p_j^*|| = ||q_i - q_j||$, for all $i,j \in\mathcal{V}$.} it is always possible to construct a positive semi-definite Laplacian matrix $L$ satisfying \eqref{eq: w_condition}. This rigidity condition requires the number of agents $n$ satisfies $n \geq m + 2$, where $m\in\mathbb{N}$ denotes the spatial dimension. We refer to \cite{kelly2014class} for techniques to construct such frameworks in 2D, and to \cite{zhao2018affine} for methods to compute the corresponding weights. If the framework is instead only \emph{globally rigid} (a weaker condition) \cite{anderson2008rigid}, it is still possible to choose the weights to satisfy \eqref{eq: w_condition} while ensuring that $L$ remains symmetric (real eigenvalues). However, to guarantee that the nonzero eigenvalues of $L$ are all positive, it becomes necessary to introduce gains $k_i\in\mathbb{R}\setminus\{0\}$ for each agent $i$ that modifies \eqref{eq: u_control} as
\begin{equation*}
    u_i = -k_i\sum_{j\in\mathcal{N}_i} w_{ij} z_{ij},
\end{equation*}
or, equivalently, considering \eqref{eq: dyn}, in compact form as
\begin{equation} \label{eq: L_control_K}
    \dot p = -KLp,
\end{equation}
where $K := \operatorname{diag}\{k_1,\dots,k_n\}$. An invertible $K$ always exists such that all nonzero eigenvalues of $KL$ have strictly positive real parts \cite{lin2016necessary, realK}, ensuring exponential convergence to a static configuration within $\mathcal{S}$, though finding such $K$ can be computationally hard \cite{realK}.

Importantly, while $KL$ is generally non-symmetric with potentially complex eigenvalues, the stability of \eqref{eq: L_control_K} is guaranteed as long as all nonzero eigenvalues of $KL$ lie in the right-half complex plane. Furthermore, since $K$ is diagonal and nonzero, constraints \eqref{eq: w_condition} remain satisfied for $KL$, and the kernel is not enlarged, so $\operatorname{Ker}\{KL\} = \operatorname{Ker}\{L\} = \mathcal{S}$. Thus, convergence to the desired shape $\mathcal{S}$ is ensured.


\section{Modified Laplacian matrix and affine maneuvering} \label{sec: Lmod}

\subsection{Modified Laplacian}

In this section, we will show how to modify a (non-unique) subset of weights $w_{ij}$ in \eqref{eq: w_condition} such that, given an arbitrary initial configuration $p(0)$, the collective motion converges to a velocity within $\mathcal{M}$ as $t\to\infty$, while $p(t)$ also converges to the desired shape $\mathcal{S}$.

Let us consider the following modified weights
\begin{equation} \label{eq: wmod}
    \tilde w_{ij} = h w_{ij} - \kappa \mu_{ij}, \quad (i,j) \in \mathcal{E},
\end{equation}
where the \emph{motion parameters} $\mu_{ij}\in\mathbb{R}$ will be designed shortly in Subsection \ref{sec: mus} for the translation, rotation, scaling, and shearing of the formation, $h \in \mathbb{R}_+$ will regulate the impact of these motion parameters over the original weights, and $\kappa\in\mathbb{R}$ will adjust the speed of the collective motion. Since our maneuvering technique is distributed, then $\mu_{ij} = 0$ whenever $j\notin\mathcal{N}_i$. Moreover, unlike the symmetry condition $w_{ij} = w_{ji}$ required for static formation control \cite{zhao2018affine}, the motion parameters generally satisfy $\mu_{ij} \neq \mu_{ji}$.


\begin{remark}[Software-Hardware Equivalence Principle]
The \emph{software parameter} $\kappa\mu_{ij}$ in \eqref{eq: wmod} is equivalent to a constant scaling $s_{ij}\in\mathbb{R}$ applied directly to the relative measurement $z_{ij}$ in \eqref{eq: u_control}, a \emph{hardware parameter}, since $\tilde w_{ij} z_{ij} = hw_{ij}(s_{ij}z_{ij})$ with $s_{ij} = 1 - \kappa\mu_{ij}/(hw_{ij})$. As $\mu_{ij}=0$ recovers $s_{ij}=1$ (a perfect measurement), each motion parameter acts as an \emph{engineered imperfection} deliberately injected into the measurements required by \eqref{eq: u_control} \cite{de2016distributed, de2020maneuvering}.
\end{remark}

Similarly to the incidence matrix $B$ in \eqref{eq: B}, consider again the ordered set of edges $\mathcal{Z}$, and define the components of the matrix $M\in\mathbb{R}^{|\mathcal{V}|\times |\mathcal{Z}|}$ as
\begin{equation} 
	\label{eq: m}
    m_{ik} :=
    \begin{cases}
        \begin{array}{ll}
            \mu_{i \mathcal{Z}_k^{\text{head}}} & \text{if} \quad i = \mathcal{Z}_k^{\text{tail}}\\
            -\mu_{i \mathcal{Z}_k^{\text{tail}}} & \text{if} \quad i = \mathcal{Z}_k^{\text{head}}\\
            0 & \text{otherwise}.
        \end{array}
    \end{cases}
\end{equation}

This definition allows us to express the \emph{modified Laplacian matrix} from the modified weights in \eqref{eq: wmod} in compact form as
\begin{equation} \label{eq: L_mod}
    \tilde L = hL - \kappa K^{-1} M B^\top.
\end{equation}
By substituting this modified Laplacian into \eqref{eq: L_control_K}, the closed-loop system becomes
\begin{equation} \label{eq: L_mod_control}
    \dot p = -K\tilde L p = -hKLp + \kappa M B^\top p,
\end{equation}
where the first term $-hKLp$ preserves the original shape-stabilizing dynamics from \eqref{eq: L_control_K}, while the new term $\kappa M B^\top p$ introduces motion control that alters the system's convergence behavior. The resulting stability and convergence to the desired shape now depend on the design of the motion parameters $\mu_{ij}$ as well as the scalar gains $h$ and $\kappa$.





\subsection{Motion parameters and reference collective motion design} \label{sec: mus}

In the following technical result, we show how to design $M$ so that $\mathcal{S}$ remains invariant if $p(0) \in \mathcal{S}$ initially. Subsequently, in Subsection \ref{sec: stability}, we will show that this design together with a lower bound for $h$ in (\ref{eq: L_mod_control}) guarantees that $p(t) \rightarrow \mathcal{S}$ and $\dot p(t) \rightarrow \mathcal{M}$ as $t \rightarrow \infty$.

\begin{lemma} \label{lem: M}
	Let $p^*$ and $L$ be as in Lemma \ref{lem: L_eigen}, and $v^* = T_\Delta(p^*)$. Consider the dynamics \eqref{eq: L_mod_control} with $p(0) \in \mathcal{S}$; if the matrix $M$ is designed so that
    \begin{equation} \label{eq: MB_pstar_vstar}
        M B^\top p^* = v^*,
    \end{equation}
    then $\mathcal{S}$ is invariant $\forall t \geq 0$.
\end{lemma}
\begin{proof}
    When $p \in \mathcal{S}$, since $p^*$ and $L$ are given as in Lemma \ref{lem: L_eigen}, we have that the term $-hKLp = 0$ in \eqref{eq: L_mod_control}. Regarding the second term in \eqref{eq: L_mod_control}, the condition \eqref{eq: MB_pstar_vstar} and $MB^\top \mathbf{1}_n = 0$ imply that, for any affine transformation $T_{\Delta'}$, we have
    \begin{align} \label{eq: MB_p_hat}
        \scalemath{0.9}{MB^\top T_{\Delta^\prime}(p^*)} & = \scalemath{0.93}{c_1^\prime MB^\top\mathbf{1}_n + c_2^\prime \operatorname{Re}(MB^\top p^*) + c_3^\prime \operatorname{Im}(MB^\top p^*)}\nonumber\\
        &\scalemath{0.9}{= c_2^\prime \operatorname{Re}(v^*) + c_3^\prime \operatorname{Im}(v^*) \nonumber}\\ 
        &\scalemath{0.9}{= T_{\Delta^\prime}(v^*) - c_1^\prime \mathbf{1}_n}.
    \end{align}
    By Lemma \ref{lem: addition}, it follows that $MB^\top T_{\Delta^\prime}(p^*) \in \mathcal{M}$, and hence $v(t) \in \mathcal{M}$ whenever $p(t) \in \mathcal{S}$. Therefore, since $p(0) \in \mathcal{S}$ and $v^* = T_\Delta(p^*)$, Lemma \ref{lem: invariance} ensures that $\mathcal{S}$ is invariant $\forall t \geq 0$.
\end{proof}

\begin{remark}
    Note that the subtraction of $c_1^\prime \mathbf{1}_n$ in \eqref{eq: MB_p_hat} reflects the physical interpretation that agent velocities are determined solely based on relative positions, i.e., $v = \kappa MB^\top p$. However, this does not imply that translations are suppressed altogether. Indeed, evaluating $T_{\Delta^\prime}(T_\Delta(p^*))$ in \eqref{eq: MB_p_hat}, as done in \eqref{eq: T_delta_T_delta}, introduces a term proportional to $\mathbf{1}_n$ that does not vanish after subtracting $c_1^\prime \mathbf{1}_n$. This point will be further clarified in the subsequent section.
\end{remark}

Similarly as in \eqref{eq: w_condition}, to ensure that \eqref{eq: MB_pstar_vstar} is satisfied, the motion parameters $\mu_{ij}$ have to satisfy the linear constraints
\begin{equation} \label{eq: v_i_body}
     \sum_{j\in\mathcal{N}_i} \mu_{ij}(p_i^* - p_j^*) = \sum_{j\in\mathcal{N}_i} \mu_{ij} z_{ij}^* = v_i^*, \quad \forall i \in \mathcal{V},
\end{equation}
where $v_i^* \in \mathbb{C}$ is the reference velocity for the agent $i$. Note that, in order to find the $\mu_{ij}$'s that satisfy \eqref{eq: v_i_body} for an arbitrary $v_i^*$, it is sufficient for the agent $i$ to have at least $m$ neighbors with the corresponding $z_{ij}^*$ being linearly independent, where $m\in\mathbb{N}$ is the spatial dimension. Indeed, this is the case if $p^*$ is \emph{generic} and the framework is \emph{globally rigid}. Notably, these same conditions are also necessary for the feasibility of the formation maintenance constraints in \eqref{eq: w_condition}. Indeed, \eqref{eq: v_i_body} is less restrictive, as it does not impose the symmetry condition $w_{ij} = w_{ji}$ required for static formation weights \cite{zhao2018affine}. In contrast, the motion parameters $\mu_{ij}$ are generally asymmetric, i.e., $\mu_{ij} \neq \mu_{ji}$. This asymmetry allows each agent $i$ to solve \eqref{eq: v_i_body} independently, further supporting a \textit{distributed implementation}.

\begin{remark}
    We remind that throughout this paper, the weights $w_{ij}$ in (\ref{eq: u_control}) and the motion parameters $\mu_{ij}$ in (\ref{eq: m}) are elements of $\mathbb{R}$, and only the positions and velocities are encoded as elements of $\mathbb{C}$ for the sake of facilitating an easier analysis of the affine formation control problem. Ultimately, the eventual implementation of our technique uses only real numbers, including the positions on the plane.
\end{remark}

\begin{figure}
    \centering
    \begin{tikzpicture}[join=round, scale=0.95]
        \begin{scope}[shift={(-5.25,0)}]
            \draw[dashed](1.5,0)--(0,0)--(0,1.5)--(1.5,1.5)--(1.5,0)--(0,1.5);
            \draw[dashed](0,0)--(1.5,1.5);
            \filldraw(0,0) circle (2pt);
            \filldraw(1.5,0) circle (2pt);
            \filldraw(1.5,1.5) circle (2pt);
            \filldraw(0,1.5) circle (2pt);
            \draw[draw=black,arrows=->](.75,.75)--(1.125,.75);
            \draw[draw=black,arrows=->](.75,.75)--(.75,1.125);
            \draw[draw=black,arrows=->](.6,-0.75)--(.1,-0.75);
            \draw[draw=black,arrows=->](.6,-0.75)--(.6,-0.25);
                \node at (.8,.45) {\small $O_b$ \normalsize};\node at (.9,-0.7) {\small $O_g$ \normalsize};\node at (-.25,-.25) {\small $p_1^*$ \normalsize}; \node at (1.25,1.75) {\small $p_3^*$ \normalsize};\node at (-.25,1.75) {\small $p_2^*$ \normalsize}; \node at (1.25,-.25) {\small $p_4^*$ \normalsize};
        \end{scope}
        \begin{scope}[shift={(-3,0)},scale=0.35]
            \filldraw(0,0) circle (2pt);
            \filldraw(1.5,0) circle (2pt);
            \filldraw(1.5,1.5) circle (2pt);
            \filldraw(0,1.5) circle (2pt);
            \draw[dashed](1.5,0)--(0,0)--(0,1.5)--(1.5,1.5)--(1.5,0)--(0,1.5);
            \draw[dashed](0,0)--(1.5,1.5);
            \draw[draw=black,color=blue,arrows=->](0,0)--(1,0);
            \draw[draw=black,color=blue,arrows=->](1.5,0)--(2.5,0);
            \draw[draw=black,color=blue,arrows=->](1.5,1.5)--(2.5,1.5);
            \draw[draw=black,color=blue,arrows=->](0,1.5)--(1,1.5);
            \node[color=blue, text width=1cm] at (0.2,-1) {\scriptsize $M_{v_x}B^\top p^*$};
        \end{scope}
        \begin{scope}[shift={(-3,1.25)},scale=0.35]
            \filldraw(0,0) circle (2pt);
            \filldraw(1.5,0) circle (2pt);
            \filldraw(1.5,1.5) circle (2pt);
            \filldraw(0,1.5) circle (2pt);
            \draw[dashed](1.5,0)--(0,0)--(0,1.5)--(1.5,1.5)--(1.5,0)--(0,1.5);
            \draw[dashed](0,0)--(1.5,1.5);
            \draw[draw=black,color=blue,arrows=->](0,0)--(0,1);
            \draw[draw=black,color=blue,arrows=->](1.5,0)--(1.5,1);
            \draw[draw=black,color=blue,arrows=->](1.5,1.5)--(1.5,2.5);
            \draw[draw=black,color=blue,arrows=->](0,1.5)--(0,2.5);
            \node[color=blue, text width=1cm] at (0.2,3.25) {\scriptsize $M_{v_y}B^\top p^*$};
        \end{scope}
        \begin{scope}[shift={(-1.5,1.25)},scale=0.35]
            \filldraw(0,0) circle (2pt);
            \filldraw(1.5,0) circle (2pt);
            \filldraw(1.5,1.5) circle (2pt);
            \filldraw(0,1.5) circle (2pt);
            \draw[dashed](1.5,0)--(0,0)--(0,1.5)--(1.5,1.5)--(1.5,0)--(0,1.5);
            \draw[dashed](0,0)--(1.5,1.5);
            \draw[draw=black,color=applegreen,arrows=->](0,0)--(-1,0);
            \draw[draw=black,color=applegreen,arrows=->](1.5,0)--(2.5,0);
            \draw[draw=black,color=applegreen,arrows=->](1.5,1.5)--(2.5,1.5);
            \draw[draw=black,color=applegreen,arrows=->](0,1.5)--(-1,1.5);
            \node[color=applegreen, text width=1cm] at (0.1,2.5) {\scriptsize $M_{v_{a_x}}B^\top p^*$};
        \end{scope}
        \begin{scope}[shift={(-1.5,0)},scale=0.35]
            \filldraw(0,0) circle (2pt);
            \filldraw(1.5,0) circle (2pt);
            \filldraw(1.5,1.5) circle (2pt);
            \filldraw(0,1.5) circle (2pt);
            \draw[dashed](1.5,0)--(0,0)--(0,1.5)--(1.5,1.5)--(1.5,0)--(0,1.5);
            \draw[dashed](0,0)--(1.5,1.5);
            \draw[draw=black,color=applegreen,arrows=->](0,0)--(0,-1);
            \draw[draw=black,color=applegreen,arrows=->](1.5,0)--(1.5,-1);
            \draw[draw=black,color=applegreen,arrows=->](1.5,1.5)--(1.5,2.5);
            \draw[draw=black,color=applegreen,arrows=->](0,1.5)--(0,2.5);
            \node[color=applegreen, text width=1cm] at (0.1,-2) {\scriptsize $M_{v_{a_y}}B^\top p^*$};
        \end{scope}
        \begin{scope}[shift={(.4,1.25)},scale=0.35]
        \filldraw(0,0) circle (2pt);
        \filldraw(1.5,0) circle (2pt);
        \filldraw(1.5,1.5) circle (2pt);
        \filldraw(0,1.5) circle (2pt);
        \draw[dashed](1.5,0)--(0,0)--(0,1.5)--(1.5,1.5)--(1.5,0)--(0,1.5);
        \draw[dashed](0,0)--(1.5,1.5);
        \draw[draw=black,color=red,arrows=->](0,0)--(-1,0);
        \draw[draw=black,color=red,arrows=->](1.5,0)--(0.5,0);
        \draw[draw=black,color=red,arrows=->](1.5,1.5)--(2.5,1.5);
        \draw[draw=black,color=red,arrows=->](0,1.5)--(1,1.5);
        \node[color=red, text width=1cm] at (-0.4,2.5) {\scriptsize $M_{v_{h_x}}B^\top p^*$};
        \end{scope}
        \begin{scope}[shift={(.4,0)},scale=0.35]
        \filldraw(0,0) circle (2pt);
        \filldraw(1.5,0) circle (2pt);
        \filldraw(1.5,1.5) circle (2pt);
        \filldraw(0,1.5) circle (2pt);
        \draw[dashed](1.5,0)--(0,0)--(0,1.5)--(1.5,1.5)--(1.5,0)--(0,1.5);
        \draw[dashed](0,0)--(1.5,1.5);
        \draw[draw=black,color=red,arrows=->](0,0)--(0,-1);
        \draw[draw=black,color=red,arrows=->](1.5,0)--(1.5,1);
        \draw[draw=black,color=red,arrows=->](1.5,1.5)--(1.5,2.5);
        \draw[draw=black,color=red,arrows=->](0,1.5)--(0,0.5);
        \node[color=red, text width=1cm] at (-0.4,-2) {\scriptsize $M_{v_{h_y}}B^\top p^*$};
        \end{scope}
        \begin{scope}[shift={(1.5,1)},scale=1]
            \draw[dashed](0,-1.75)--(0,1.5);
        \end{scope}
        \begin{scope}[shift={(2.1,0)},scale=0.35]
            \filldraw(0,0) circle (2pt);
            \filldraw(1.5,0) circle (2pt);
            \filldraw(1.5,1.5) circle (2pt);
            \filldraw(0,1.5) circle (2pt);
            \draw[dashed](1.5,0)--(0,0)--(0,1.5)--(1.5,1.5)--(1.5,0)--(0,1.5);
            \draw[dashed](0,0)--(1.5,1.5);
            \draw[draw=black,color=red,arrows=->](0,0)--(.75,-.75);
            \draw[draw=black,color=red,arrows=->](1.5,0)--(2.25,.75);
            \draw[draw=black,color=red,arrows=->](1.5,1.5)--(0.75,2.25);
            \draw[draw=black,color=red,arrows=->](0,1.5)--(-0.75,0.75);
            \node[color=red, text width=1cm] at (0.4,-2) {\scriptsize $M_{v_\omega}B^\top p^*$};
        \end{scope}
        \begin{scope}[shift={(2.1,1.25)},scale=0.35]
            \filldraw(0,0) circle (2pt);
            \filldraw(1.5,0) circle (2pt);
            \filldraw(1.5,1.5) circle (2pt);
            \filldraw(0,1.5) circle (2pt);
            \draw[dashed](1.5,0)--(0,0)--(0,1.5)--(1.5,1.5)--(1.5,0)--(0,1.5);
            \draw[dashed](0,0)--(1.5,1.5);
            \draw[draw=black,color=red,arrows=->](0,0)--(-.75,-.75);
            \draw[draw=black,color=red,arrows=->](1.5,0)--(.85,.65);
            \draw[draw=black,color=red,arrows=->](1.5,1.5)--(2.25,2.25);
            \draw[draw=black,color=red,arrows=->](0,1.5)--(.65,.85);
            \node[color=red, text width=1cm] at (0.4,3.25) {\scriptsize $M_{v_s} B^\top p^*$};
        \end{scope}
        
    \end{tikzpicture}
	\caption{Square formation in 2D with $\mathcal{E}$ derived from a complete graph of four nodes, so that the associated framework is universally rigid. The reference shape $p^*$ for the affine formation control is defined relative to the body-fixed frame $\mathcal{O}_b$. Each agent’s reference velocity $v_i^*$ is constructed as a linear combination of its relative positions $z_{ij}$, $j\in\mathcal{N}_i$, employing the motion parameters $\mu_{ij}$. To construct all possible affine collective motions in 2D, we show a basis consisting of six independent group motions: two orthogonal translation velocities (shown in blue), two orthogonal scaling motions (in green), and two orthogonal shearing motions (in red). As an alternative, the two shear components can be replaced (as shown to the right of the vertical dashed line) by a \textit{cross-scaling} motion and a rotational (spinning) motion, yielding an equivalent affine basis.}
    \label{fig: affine_v}
\end{figure}
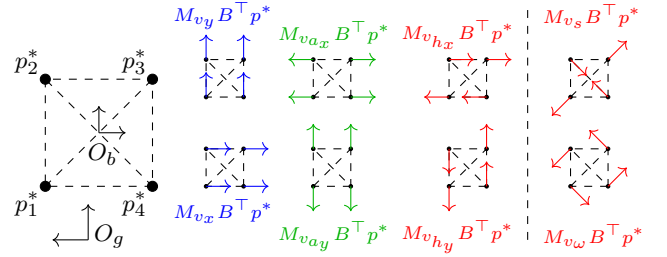

\begin{example}
    Let us show the design of $\mu_{ij}$ for the spinning motion in \autoref{fig: affine_v}. It is clear that $v_1^* = \frac{1}{\sqrt{2}}(1-\iota)$, ${v}_2^* = \frac{1}{\sqrt{2}}(-1-\iota)$, $ {v}_3^* = \frac{1}{\sqrt{2}}(-1+\iota)$ and $v_4^* = \frac{1}{\sqrt{2}}(1+\iota)$ up to an arbitrary scale (angular speed) factor. In order to satisfy \eqref{eq: v_i_body}, assuming the square has unit side length, then the motion parameters of the agent $1$ for the spinning motion are $\mu_{12} = -\mu_{14} = \frac{1}{\sqrt{2}}$, and $\mu_{13} = 0$, since $(p^*_1 - p^*_2) = -\iota$, and $(p^*_1 - p^*_4) = -1$. Note that this is not the only choice, as $(p^*_1 - p^*_3) = (-1-\iota)$ is unused.
\end{example}

Regarding the design of the reference collective motion, the only requirement according to Lemma \ref{lem: invariance} is that $v^* = T_\Delta(p^*)$. We propose the following expression
\begin{align} \label{eq: v_star}
    v^* = T_{\Delta_v}(p^*) = (v_x + \iota v_y) \mathbf{1}_n &+ (v_{a_x} + \iota v_{h_y}) \operatorname{Re}(p^*) \nonumber \\ &+ (v_{h_x} + \iota v_{a_y}) \operatorname{Im}(p^*),
\end{align}
with $\Delta_v = [v_x\; v_y\; v_{a_x}\; v_{a_y}\; v_{h_x}\; v_{h_y}] \in \mathbb{R}^6$. Here, $v_{\{x,y\}}$ account for horizontal and vertical translation velocity, $v_{\{a_x,a_y\}}$ for scaling parallel to the $x$- and $y$-axis of $\mathcal{O}_b$, respectively, and $v_{\{h_x,h_y\}}$ for the shearing velocity along the same axes. 

Similarly, the matrix $M$ in \eqref{eq: MB_pstar_vstar} can be decomposed into six components, namely
\begin{align} \label{eq: M_split}
		M = v_x M_{v_x} + v_y M_{v_y} &+ v_{a_x} M_{v_{a_x}} + v_{a_y} M_{v_{a_y}} \nonumber\\
          & + v_{h_x} M_{v_{h_x}} + v_{h_y} M_{v_{h_y}},
\end{align}
where each matrix $M_{\{\cdot\}}  \in \mathbb{R}^{|\mathcal{V}|\times |\mathcal{Z}|}$ is constructed using the same formulation as in \eqref{eq: m}, but tailored to its corresponding component of the reference collective motion, as in \autoref{fig: affine_v}.

Importantly, the basis matrices in \eqref{eq: M_split} are precomputed independently by each agent $i$ from \eqref{eq: v_i_body}, depending solely on $(\mathcal{G},p^*)$ and requiring recomputation only if these change. The design is thus decoupled from the high-level command $\Delta_v$: once $(\mathcal{G},p^*)$ is fixed, transitions between geometries reduce to updating $\Delta_v$, without re-solving \eqref{eq: v_i_body} or modifying $p^*$ or $L$, yielding a modular architecture suited to frequent maneuver updates. Moreover, since all admissible configurations lie in the invariant subspace $\mathcal{S}$, agents may adopt new $\Delta_v$ values asynchronously: our convergence guarantees ensure stability across transitions, with timing affecting only the convergence rate.

\begin{remark}[Scope of the distributed implementation]
The control law \eqref{eq: L_mod_control} runs on each agent using only local measurements $z_{ij}$, with $\mu_{ij}$ computed locally from $(\mathcal{G},p^*)$; the runtime control is thus fully distributed, with no leader and no inter-agent state tracking. The only shared quantity is $\Delta_v$, needed to evaluate $v_i^*$ in \eqref{eq: v_i_body}: not a feedback signal but a high-level operational command, analogous to a mission parameter, that selects the maneuver and sits one layer above the distributed loop. This contrasts with leader-follower schemes, where the shared information is a state tracked continuously inside the loop; here $\Delta_v$ is constant between switches and carries only six real numbers, so it can be preloaded, broadcast, or agreed by consensus over $\mathcal{G}$ between switches.
\end{remark}

\begin{remark}
    Note that the basis of six orthogonal motions in \eqref{eq: v_star} is not unique. For engineering or application purposes, one might, for example, consider the rotational motion $v_\omega = v_{h_y} - v_{h_x}$ and the combined shearing motion $v_s = v_{h_y} + v_{h_x}$ instead of the two orthogonal shearing motions $v_{\{h_x, h_y\}}$. These two alternative bases are illustrated in \autoref{fig: affine_v}, where $M_{v_\omega} = M_{v_{h_y}} - M_{v_{h_x}}$ and $M_{v_s} = M_{v_{h_y}} + M_{v_{h_x}}$.
\end{remark}


\subsection{Analytical solution of the stationary motion} \label{sec: anal_solution}

The analytical solution of \eqref{eq: L_mod_control} is given by
\begin{equation} \label{eq: pt_exp}
    p(t) = e^{-K\tilde Lt}p(0),
\end{equation}
where $e^{-K\tilde Lt}$ denotes the matrix exponential of $-K\tilde Lt$. 
Using the Jordan normal form of $-K\tilde L$, the solution becomes
\begin{equation} \label{eq: pt_exp_jordan}
    p(t) = \sum_{k=1}^n \alpha_k f_k e^{\lambda_k t},
\end{equation}
where $f_k(t,x_k,x_{k-1},\dots,x_{k-g})$, with $g \in \mathbb{N}$ depending on the algebraic and geometric multiplicity of the eigenvalue $\lambda_k \in \mathbb{C}$, are different functions corresponding to linear combinations $(x_k + x_{k-1}t + x_{k-2}\frac{t^2}{2!} + \cdots + x_{l-g}\frac{t^g}{g!})$, depending on the (possibly generalized) eigenvectors $x_k \in \mathbb{C}^n$ and constants $\alpha_k \in \mathbb{C}$ given by the initial condition $p(0)$. Thus, deriving the eigenstructure of $-K\tilde L$ allows explicit expression of the analytical solution of \eqref{eq: L_mod_control}.


Before proving convergence $p(t) \to \mathcal{S}$ for sufficiently large $h$ in \eqref{eq: L_mod_control}, we analyze the solution when $p(0) \in \mathcal{S}$, corresponding to identifying the eigenvalues of $K \tilde L$ whose eigenvectors belong to $\mathcal{S}$. To obtain these eigenvalues, we consider a generic affine transformation $T_\Delta$ from \eqref{eq: p_hat}, the matrix $M$ from Lemma \ref{lem: M}, and the reference collective motion $v^*$ in \eqref{eq: v_star}. Then, using \eqref{eq: MB_p_hat} and simplifying \eqref{eq: T_delta_T_delta} via \eqref{eq: cs}, we obtain
\begin{align} \label{eq: MB_phat_ker}
    MB^\top T_\Delta(p^*) &= T_\Delta(v^*) - c_1 \mathbf{1}_n \nonumber\\
    & = \left[ v_x c_2 + v_y c_3 \right] \mathbf{1}_n \nonumber\\
    & + \left[ v_{a_x} c_2 + v_{h_y} c_3 \right] \operatorname{Re}(p^*)\nonumber\\
    &  + \left[ v_{a_y} c_3 + v_{h_x} c_2 \right] \operatorname{Im}(p^*),
\end{align}
where $c_1,c_2,c_3$ parametrize the affine transformation encoded by $\Delta$. Now, considering the eigenvalue equation for $-K\tilde L$ with corresponding eigenvectors within $\mathcal{S}$, we have that
\begin{equation} \label{eq: eig_eq}
    -K\tilde L x_\lambda = \kappa MB^\top x_\lambda = \lambda x_\lambda,
\end{equation}
where $x_\lambda \in \mathcal{S}$ is the eigenvector associated with the eigenvalue $\lambda \in \mathbb{C}$. Substituting \eqref{eq: MB_phat_ker} in \eqref{eq: eig_eq}, with $T_\Delta(p^*) = x_\lambda$, we obtain the following three linear eigenvalue constraints
\begin{align} \label{eq: eig_constraints}
    \underbrace{
    \left[\begin{matrix}
        0 & v_x & v_y \\
        0 & v_{a_x} & v_{h_y} \\
        0 & v_{h_x} & v_{a_y} \\
    \end{matrix}\right]
    }_{A_{v^*}}
    \underbrace{
    \left[\begin{matrix}
        c_1 \\ c_2 \\ c_3
    \end{matrix}\right]
    }_{[x_\lambda]_{\mathcal{B}_\mathcal{S}}} =
    \frac{\lambda}{\kappa}
    \left[\begin{matrix}
        c_1 \\ c_2 \\ c_3
    \end{matrix}\right].
\end{align}

\begin{remark}
    The constraints in \eqref{eq: eig_constraints} reduce the problem of determining eigenvalues $\lambda$ of $-K\tilde L$ with eigenvectors $x_\lambda \in \mathcal{S}$ to finding eigenvalues and eigenvectors of matrix $A_{v^*}$, significantly reducing complexity. This leads to the following technical result analyzing all solutions of the eigenvector/eigenvalue problem depending on the design of $v^*$ in \eqref{eq: v_star}.
\end{remark}

\begin{table*}[t]\centering
\caption{
Cases of Prop.~\ref{prop: eigen_deg}/Thm.~\ref{thm: analytical_sol}, organized by the linear-part generator $G=\big[\begin{smallmatrix}v_{a_x}&v_{h_y}\\ v_{h_x}&v_{a_y}\end{smallmatrix}\big]$ (eigenvalues $l_\pm/\kappa$) and the translation $w=[v_x,v_y]^\top$. When $\det G=0$, $\eta=[v_{h_y}, -v_{a_x}]$ spans $\ker G$ and $r:=w^\top\eta=v_xv_{h_y}-v_yv_{a_x}$ tests whether translation excites the shape direction $G$ holds fixed. Since $\dot c_1=\kappa\,w^\top[c_2,c_3]^\top$ from \eqref{eq: MB_phat_ker}, each resonant coupling ($r\neq0$) forces an integrator and adds one power of $t$.
}
\vspace{1mm}
\label{tab: cases}
\begin{tabular}{cllll}
\toprule
Case & Deformation ($G$) & Translation coupling & Modes & Resulting collective (affine) motion \\
\midrule
C1) & $\det G\neq0$, diagonalizable & absorbed ($G$ invertible) & $1,\,e^{l_+t},\,e^{l_-t}$ & \emph{Generic}: scaling and/or rotation; translation re-centers \\
C2) & $\det G\neq0$, defective & absorbed ($G$ invertible) & $1,\,e^{lt},\,t\,e^{lt}$ & Uniform scaling with a superposed single-axis shear \\
C3) & $\det G=0,\ \operatorname{tr}G\neq0$ & compatible: $r=0$ & $1,\,1,\,e^{lt}$ & Scaling mode $+$ steady drift; integrator not forced \\
C4) & $\det G=0,\ \operatorname{tr}G\neq0$ & resonant: $r\neq0$ & $1,\,t,\,e^{lt}$ & Scaling mode $+$ ramp drift; integrator forced once \\
C5) & $\det G=0,\ \operatorname{tr}G=0$ & compatible: $r=0$ (or $G=0$) & $1,\,t$ & Constant-rate translation and/or shear (single integration) \\
C6) & $\det G=0,\ \operatorname{tr}G=0$ & resonant: $r\neq0$ & $1,\,t,\,t^2$ & Shear drift integrated again by resonant translation (double) \\
\bottomrule
\end{tabular}
\end{table*}

\begin{prop} \label{prop: eigen_deg}
	Given $p^*$ and $L$ as in Lemma \ref{lem: L_eigen} and $M$ as in Lemma \ref{lem: M}, the three eigenvalues of $-K\tilde L$ with corresponding eigenvectors within $\mathcal{S}$ are $\{0, l_+, l_-\}$ where
   \begin{align}
        l_{\pm} = \; & \kappa\left(\scalemath{0.8}{\frac{v_{a_x} + v_{a_y}}{2}} + \sigma_{\pm}\right) \nonumber \\
        = \; & \kappa\left(\scalemath{0.8}{\frac{v_{a_x} + v_{a_y}}{2} \pm \sqrt{\left(\frac{v_{a_x} - v_{a_y}}{2}\right)^2 + v_{h_x} v_{h_y}}} \, \right). \nonumber
    \end{align}
    Regarding the eigenvectors, consider the following six mutually exclusive cases, based on the design of $v^*$ in \eqref{eq: v_star}:
    \begin{enumerate}
        
        \item [C1)] If $l_\pm \neq 0$ and, in the case where $l_+ = l_-$, it holds that $v_{a_x} = v_{a_y}$ and $h_x = h_y = 0$, then the eigenvector of the zero eigenvalue is given by $\mathbf{1}_n$. The eigenvectors corresponding to $l_\pm$ are given by
        \begin{align} \label{eq: c1_lpm}
            \scalemath{0.85}{
            x_{l_\pm} = \gamma_{\pm}\mathbf{1}_n + v_{h_y}\operatorname{Re}(p^*) + \left(\frac{l_{\pm}}{\kappa} - v_{a_x}\right) \operatorname{Im}(p^*)},
        \end{align}
        with $\gamma_{\pm} = \frac{v_x v_{h_y} + v_y (l_{\pm}/\kappa - v_{a_x})}{l_{\pm}/\kappa}$, except when $l_+ = l_-$, $v_{a_x} = v_{a_y}$ and $v_{h_x} = v_{h_y} = 0$. In this case, the eigenvectors are given by
        \begin{align}
            &\scalemath{0.85}{
            x_{l_+} = v_x \mathbf{1}_n + \left(\frac{v_{a_x} + v_{a_y}}{2}\right)\operatorname{Re}(p^*)}, \label{eq: c1_lp}\\
            &\scalemath{0.85}{
            x_{l_-} = v_y \mathbf{1}_n + \left(\frac{v_{a_x} + v_{a_y}}{2}\right)\operatorname{Im}(p^*)}. \label{eq: c1_lm}
        \end{align}

        \item [C2)] If $l_\pm \neq 0$, with $l_+ = l_-$, and it holds that $v_{a_x} = v_{a_y}$ and either $v_{h_x} = 0$ or $v_{h_y} = 0$, then the eigenvector corresponding to the zero eigenvalue is given by $\mathbf{1}_n$. For the $l = l_+ = l-$ eigenvalue, the geometric multiplicity is 1, and its associated chain of generalized eigenvectors is $\{x_l^2, x_l^1\}$, with
        \begin{equation} \label{eq: c2_xl_1}
            \scalemath{0.77}{
            x_l^1 = 
            \begin{cases}
                \begin{array}{rl}
                    \kappa [v_x \mathbf{1}_n + \frac{v_{a_x} + v_{a_y}}{2}\operatorname{Re}(p^*)] & \text{if} \quad v_{h_x} = 0\\
                    \kappa [v_y \mathbf{1}_n + \frac{v_{a_x} + v_{a_y}}{2}\operatorname{Im}(p^*)] & \text{if} \quad v_{h_y} = 0,
                \end{array}
            \end{cases}}
        \end{equation}
        \begin{equation} \label{eq: c2_xl_2}
            \scalemath{0.77}{
            x_l^2 = 
            \begin{cases}
                \begin{array}{rl}
                    \frac{v_y}{v_{h_y}} \mathbf{1}_n + \operatorname{Re}(p^*) + \frac{v_{a_x} + v_{a_y}}{2 v_{h_y}}\operatorname{Im}(p^*) & \text{if} \quad v_{h_x} = 0\\
                    \frac{v_x}{v_{h_x}} \mathbf{1}_n + \frac{v_{a_x} + v_{a_y}}{2v_{h_x}}\operatorname{Re}(p^*) + \operatorname{Im}(p^*) & \text{if} \quad v_{h_y} = 0. 
                \end{array}
            \end{cases}}
        \end{equation}

        \item [C3)] If one of the $l_\pm$ eigenvalues is equal to 0, and either $v_{h_y}v_x = v_{a_x} v_y$ or $v_{h_x} v_y = v_{a_y} v_x$, the eigenvectors corresponding to the nonzero $l_\pm$ remains with the same eigenvector as in \eqref{eq: c1_lpm}, and the ones corresponding to the zero eigenvalues are $\mathbf{1}_n$ and
        \begin{align}
            x_0 &= v_y \operatorname{Re}(p^*) - v_x \operatorname{Im}(p^*). \label{eq: c3_x0}
        \end{align}

        \item [C4)] If one of the $l_\pm$ eigenvalues is equal to 0, and either $v_{h_y} v_x \neq v_{a_x} v_y$ or $v_{h_x} v_y \neq v_{a_y} v_x$, the eigenvectors corresponding to the nonzero $l_\pm$ remains with the same eigenvector as in \eqref{eq: c1_lpm}. For the zero eigenvalues, the geometric multiplicity is 1, and the associated chain of generalized eigenvectors is $\{x_0^2, x_0^1\}$, with
        \begin{align}
            x_0^1 &= \kappa [v_x v_{h_y} - v_{a_x} v_y] \mathbf{1}_n, \label{eq: c4_x0_1} \\
            x_0^2 &= v_{h_y} \operatorname{Re}(p^*) - v_{a_x} \operatorname{Im}(p^*). \label{eq: c4_x0_2}
        \end{align}

        \item [C5)] If $l_+,l_- = 0$, and either $v_{h_y} = 0$ with $v_x = 0$, $v_{h_x} = 0$ with $v_y = 0$, or $v_{h_x}, v_{h_y} = 0$ with either $v_x \neq 0$ or $v_y \neq 0$, then the geometric multiplicity is 2. The chain of generalized eigenvectors associated with the degenerated eigenvalues is $\{x_0^2,x_0^1\}$, with
        \begin{equation} \label{eq: c5_x01}
            \scalemath{0.8}{
            x_0^1 = 
            \begin{cases}
                \begin{array}{rl}
                    \kappa [v_y \mathbf{1}_n + v_{h_y}\operatorname{Re}(p^*)] & \text{if} \quad v_{h_x},v_x = 0\\
                    \kappa [v_x \mathbf{1}_n + v_{h_x}\operatorname{Im}(p^*)] & \text{if} \quad v_{h_y},v_y = 0\\
                    \kappa (v_x + \iota v_y) \mathbf{1}_n & \text{if} \quad 
                    \begin{aligned}
                        &v_{h_x}, v_{h_y} = 0 \\
                        &v_x \neq 0 \text{ or } v_y \neq 0.
                    \end{aligned}
                \end{array}
            \end{cases}}
        \end{equation}
        \begin{equation} \label{eq: c5_x02}
            \scalemath{0.8}{
            x_0^2 = 
            \begin{cases}
                \begin{array}{rl}
                    \operatorname{Im}(p^*) & \text{if} \quad v_{h_x}, v_x = 0\\
                    \operatorname{Re}(p^*) & \text{if} \quad v_{h_y}, v_y = 0\\                 
                    v_x \operatorname{Re}(p^*) + \iota v_y \operatorname{Im}(p^*)
                    & \text{if} \quad 
                    \begin{aligned}
                        &v_{h_x},v_{h_y} = 0 \\
                        &v_x \neq 0 \text{ or } v_y \neq 0.
                    \end{aligned}
                \end{array}
            \end{cases}}
        \end{equation}
        and for the non-degenerated eigenvalue, the associated eigenvector is
        \begin{equation} \label{eq: c5_y0}
            \scalemath{0.8}{
            y_0 = 
            \begin{cases}
                \begin{array}{rl}
                    v_{h_y} \mathbf{1}_n - v_y \operatorname{Re}(p^*) & \text{if} \quad v_{h_x}, v_x = 0\\
                    v_{h_x}\mathbf{1}_n - v_x \operatorname{Im}(p^*) & \text{if} \quad v_{h_y},v_y = 0\\
                    v_y \operatorname{Re}(p^*) - v_x \operatorname{Im}(p^*) & \text{if} \quad 
                    \begin{aligned}
                        &v_{h_x}, v_{h_y} = 0 \\
                        &v_x \neq 0 \text{ or } v_y \neq 0.
                    \end{aligned}
                \end{array}
            \end{cases}}
        \end{equation}

        \item [C6)] If $l_+,l_- = 0$ and either $v_{h_x},v_x \neq 0$ or $v_{h_y}, v_y \neq 0$, then the geometric multiplicity is 1, and the chain of generalized eigenvectors is $\{x_0^3,x_0^2,x_0^1\}$, with
        \begin{equation} \label{eq: c6_x01}
            \scalemath{0.9}{
            x_0^1 = 
            \begin{cases}
                \begin{array}{rl}
                    \kappa^2 v_y^2 v_{h_x} \mathbf{1}_n & \text{if} \quad v_{h_x}, v_y \neq 0\\
                    \kappa^2 v_x^2 v_{h_y} \mathbf{1}_n & \text{if} \quad v_{h_y}, v_x \neq 0,
                \end{array}
            \end{cases}}
        \end{equation}
        \begin{equation} \label{eq: c6_x02}
            \scalemath{0.85}{
            x_0^2 = 
            \begin{cases}
                \begin{array}{rl}
                    \kappa [v_x v_y \mathbf{1}_n + v_y v_{h_x} \operatorname{Im}(p^*)] & \text{if} \quad v_{h_x}, v_y \neq 0\\
                    \kappa [v_x v_y \mathbf{1}_n + v_x v_{h_y} \operatorname{Re}(p^*)] & \text{if} \quad v_{h_y}, v_x \neq 0,
                \end{array}
            \end{cases}}
        \end{equation}
        \begin{equation} \label{eq: c6_x03}
            \scalemath{0.9}{
            x_0^3 = 
            \begin{cases}
                \begin{array}{rl}
                    v_y \operatorname{Re}(p^*) & \text{if} \quad v_{h_x}, v_y \neq 0\\
                    v_x \operatorname{Im}(p^*) & \text{if} \quad v_{h_y}, v_x \neq 0.
                \end{array}
            \end{cases}}
        \end{equation}
        
    \end{enumerate}
\end{prop}
\begin{proof}
  See Subsection \ref{ap: prop_eig} in the Appendix.
\end{proof}

The results in Proposition \ref{prop: eigen_deg} can be directly decoded into $\mathbb{R}^2$, as shown in Remark \ref{remark} for Lemma \ref{lem: L_eigen}. Having identified the analytical expressions of the three eigenvectors of $-K\tilde L$ spanning $\mathcal{S}$, along with their corresponding eigenvalues, we are now able to present the analytical solution of \eqref{eq: L_mod_control} describing the agents’ stationary motion.


\begin{theorem} \label{thm: analytical_sol}
    Consider $p^*$ and $L$ as in Lemma \ref{lem: L_eigen}, $M$ as in Lemma \ref{lem: M}, $v^*$ as in \eqref{eq: v_star}, and $u = -K\tilde Lp$ as the control law for the dynamics in \eqref{eq: dyn}. If $p(0) \in \mathcal{S}$, then, based on the six cases outlined in Proposition \ref{prop: eigen_deg},
    \begin{equation} \label{eq: p_par_t}
        p(t) = 
        \begin{cases} 
            \alpha_1 \mathbf{1}_n  + \alpha_2 x_{l_+}e^{l_+ t} + \alpha_3 x_{l_-}e^{l_- t} & \mbox{if } \textit{C1)} \\

            \alpha_1 \mathbf{1}_n  + [\alpha_2 x_l^1  + \alpha_3 (x_l^2 + x_l^1 t)] e^{l t} & \mbox{if } \textit{C2)} \\

            \alpha_1 \mathbf{1}_n  + \alpha_2 x_0 + \alpha_3 x_{l}e^{l t} & \mbox{if } \textit{C3)} \\

            \alpha_1 x_0^1  + \alpha_2 (x_0^2 + x_0^1 t) + \alpha_3 x_{l}e^{l t} & \mbox{if } \textit{C4)} \\

            \alpha_1 x_0^1  + \alpha_2 (x_0^2 + x_0^1 t) + \alpha_3 y_0 & \mbox{if } \textit{C5)} \\

            \alpha_1 x_0^1  + \alpha_2 (x_0^2 + x_0^1 t) \\
            \qquad \quad + \; \alpha_3 (x_0^3 + x_0^2 t + x_0^1 \frac{t^2}{2})& \mbox{if } \textit{C6)},
        \end{cases}
    \end{equation}
    for all $t \geq 0$, where $\alpha_{\{1,2,3\}} \in \mathbb{C}$ depend on $p(0)$, and $l \in \mathbb{C}$ represents either the nonzero $l_{\pm}$ eigenvalue or the unique $l_{\pm}$ eigenvalue when $l_+ = l_- \neq 0$.
\end{theorem}
\begin{proof}
    Since $\dot p = -K\tilde Lp$, from \eqref{eq: pt_exp_jordan} we have that $p(0) = \sum_{k=1}^n f_k$, where
    the fist three $f_k$ components correspond to the (potentially generalized) eigenvectors of $-K\tilde L$ that span $\mathcal{S}$. Thus, since $p(0) \in \mathcal{S}$, it can be expressed in the basis formed by these three eigenvectors, implying that all the rest of the contributions $f_k$ with $k \geq 4$ are zero. Therefore, $p(t)$ is given by \eqref{eq: p_par_t} for all $t \geq 0$. 
\end{proof}

\subsection{Stability analysis}\label{sec: stability}


In this subsection, we show that $p(t)$, as characterized by Theorem~\ref{thm: analytical_sol}, converges to the subspace $\mathcal{S}$, and that $\dot p(t)$, governed by \eqref{eq: L_mod_control}, converges to $\mathcal{M}$ as $t \rightarrow \infty$, provided that $h$ is sufficiently large and $p(0)\in\mathbb{C}^n$.

We begin by defining the complementary subspace of $\mathcal{S} = \operatorname{Ker}\{L\}$ as $\mathcal{C} := \operatorname{Img}\{L\}$. Let $P_\mathcal{S},P_{\mathcal{C}}\in\mathbb{C}^{n\times n}$ denote the projection matrices onto $\mathcal{S}$ and $\mathcal{C}$, respectively, satisfying $P_\mathcal{S} = 1 - P_\mathcal{C}$. Note that these two spaces are not orthogonal unless $L$ is Hermitian, as assumed in \cite{hgdemarina_affine}. We decompose $p$ as
\begin{equation}
	\label{eq: p_split}
	p = P_{\mathcal{S}} p + P_{\mathcal{C}} p = p_\mathcal{S} + p_\mathcal{C},
\end{equation}
which enables us to express the dynamics from \eqref{eq: L_mod_control} as
\begin{equation} \label{eq: dot_p_perp}
\dot p_\mathcal{C} = P_{\mathcal{C}} \dot p = -P_{\mathcal{C}}hKL(p_\mathcal{S} + p_\mathcal{C}) + P_{\mathcal{C}} \kappa MB^\top (p_\mathcal{S} + p_\mathcal{C}).
\end{equation}

Lemma \ref{lem: M} ensures that if $p(t_0) \in \mathcal{S}$, then $\mathcal{S}$ is invariant under the dynamics $\dot p(t) = MB^\top T_\Delta(p^*) \in \mathcal{S}, \forall t\geq t_0$. Consequently, $T_\Delta(p^*)\in\mathcal{S}$, which yields $P_\mathcal{C} MB^\top p_\mathcal{S} = 0$. Furthermore, Lemma \ref{lem: L_eigen} gives $KL p_\mathcal{S} = 0$. Applying these results, we can simplify \eqref{eq: dot_p_perp} as
\begin{align} 
	\dot p_\mathcal{C} &= -h P_{\mathcal{C}}KLp_\mathcal{C} + \kappa P_{\mathcal{C}} MB^\top p_\mathcal{C} \label{eq: dot_p_pert_simp}.
\end{align}

The following proposition provides a sufficient condition on the gain $h$ to guarantee the exponential stability of \eqref{eq: dot_p_pert_simp}.

\begin{prop} \label{prop: h}
	Let $L$ be designed as in Lemma \ref{lem: L_eigen}, $K$ so that $KL$ does not have eigenvalues with negative real component, and $M$ be designed as in Lemma \ref{lem: M}. Then, there exists a lower bound $h_l \in \mathbb{R}^+$ such that, for any $h > h_l$, the origin of system \eqref{eq: dot_p_pert_simp} is exponentially stable for all $p_\mathcal{C}(0) \in \mathcal{C}$. This is given by
	\begin{equation}\label{eq: hl}
        h_l = \kappa \, \|Q\|_2 \, \| MB^\top \|_2,
	\end{equation}
	where $Q$ is the unique positive definite matrix satisfying the Lyapunov equation $QJ_2 + J_2^H Q = 2I_{n-3}$, and $J_2 \in \mathbb{C}^{(n-3)\times (n-3)}$ denotes the Jordan form associated with the nonzero eigenvalues of $P_\mathcal{C}KL$.
\end{prop}
\begin{proof}
Firstly, we show that the origin of \eqref{eq: dot_p_pert_simp} is the only stable equilibrium 
. Differently from \cite{hgdemarina_affine}, we do not have that $KLp_\mathcal{C} \in \mathcal{C}$, so $P_{\mathcal{C}} KLp_\mathcal{C} \neq KLp_\mathcal{C}$ in general, e.g., $K$ is not proportional to the identity matrix. Nonetheless, it is straightforward to see that $-hP_{\mathcal{C}}KL$ is still a stable matrix for $p_\mathcal{C}$, i.e., for $\dot p_\mathcal{C}(t) = -h P_{\mathcal{C}}KLp_\mathcal{C}(t)$ we have that $p_\mathcal{C}(t) \to 0$ exponentially fast as $t\to\infty$ assuming $p_\mathcal{C}(0) \in \mathcal{C}$. 
This follows from the fact that $-KL$ has no eigenvalue with positive real part and its kernel is $\mathcal{S}$, while $p_\mathcal{C}(t)$ lives in the complementary subspace $\mathcal{C}$ for all time, with $P_{\mathcal{C}}KLp_\mathcal{C} \in \mathcal{C}$. The second term in \eqref{eq: dot_p_pert_simp} can be seen as a perturbation. For a sufficiently large ratio $h/\kappa$, the nonzero eigenvalues of $-K\tilde L = (-h KL + \kappa MB^\top)$ are close to the eigenvalues of $-hKL$ and far from becoming zero. In such a case, according to Theorem \ref{thm: analytical_sol}, the kernel of $(-h KL + \kappa MB^\top)$ is equal to or within $\mathcal{S}$; thus, the origin is the only stable equilibrium of \eqref{eq: dot_p_pert_simp} when $p_\mathcal{C}(0) \in \mathcal{C}$. That is, the equality $P_\mathcal{C}(-hKL + \kappa MB^\top)p_\mathcal{C} = 0$ holds if and only if $p_\mathcal{C} \in \mathcal{S}$, i.e., $p_\mathcal{C} = 0$.

Secondly, we find the lower bound for $h$ following \cite[Theorem 1]{hgdemarina_affine}. Choose a coordinate transformation $T\in\mathbb{C}^{n\times n}$ so that the complementary subspaces $\mathcal{S}$ and $\mathcal{C}$ are orthogonal. 
In this coordinate system, the system \eqref{eq: dot_p_pert_simp} can be split into two independent subsystems whose signals live only in $\mathcal{S}$ and $\mathcal{C}$, respectively. In fact, in the new coordinate system the initial condition can have the form $Tp_\mathcal{C}(0) = q(0) = \begin{bmatrix}0 & q_2(0)^\top \end{bmatrix}^\top \in\mathbb{C}^n, q_2\in\mathbb{C}^{n-3}$, and $P_\mathcal{C}$ in \eqref{eq: dot_p_pert_simp} annihilates any possible component \emph{escaping} from $\mathcal{C}$; thus, the three first components of $q(t)$ are always zero. Consequently, the Jordan form $J = T P_\mathcal{C}KL T^{-1} = \left[\begin{smallmatrix}J_1 & 0 \\ 0 & J_2\end{smallmatrix}\right]\in\mathbb{C}^{n\times n}$, where $J_1\in\mathbb{R}^{3\times 3}$ is the zero matrix and $J_2\in\mathbb{R}^{(n-3) \times (n-3)}$ is Hurwitz. The dynamics of $q_2(t)$ can be written as follows
\begin{equation} \label{eq: aux105}
\dot q_2 = -hJ_2 q_2 + \kappa \left(TP_\mathcal{C}MB^\top T^{-1}\right)^\dagger q_2,
\end{equation}
where the symbol $\dagger$ means that we take the last $(n-3) \times (n-3)$ diagonal block of the matrix to accommodate for the dimensions of $q_2$, e.g., $J^\dagger = J_2$. 

Considering the Lyapunov function $V = q_2^H Q q_2$ for \eqref{eq: aux105}, where $Q$ is a positive definite matrix satisfying $QJ_2 + J_2^H Q = 2I_{(n-3)}$, the time derivative satisfies
\begin{align*}
	\frac{\mathrm{d}V}{\mathrm{dt}} &\leq -2h\|q_2\|^2 +2\kappa \, \|Q\left(TP_{\mathcal{C}} MB^\top T^{-1}\right)^\dagger \|_2 \, \|q_2\|^2 \nonumber \\
	&\leq -2h\|q_2\|^2 +2\kappa \, \|Q\|_2 \, \|TP_{\mathcal{C}} MB^\top T^{-1} \|_2 \, \|q_2\|^2 \nonumber \\
	&\leq -2h\|q_2\|^2 +2\kappa \, \|Q\|_2 \, \|MB^\top\|_2 \, \|q_2\|^2,
\end{align*}
using $\|A\|_2^\dagger \leq \|A\|_2$, coordinate transformation norm preservation, and the fact that projection matrices do not increase the norm of a vector. Therefore, choosing $h > \kappa \|Q\|_2 \| MB^\top \|_2$ ensures $q_2(t) \to 0$ as $t \to \infty$, thus $p_\mathcal{C}(t) \to 0$ exponentially fast when $p_\mathcal{C}(0)\in\mathcal{C}$.
\end{proof}

Having established the exponential stability of $p_\mathcal{C}(t)$ to zero, i.e., $p(t) \rightarrow \mathcal{S}$, as $t\rightarrow\infty$ under the conditions of Proposition \ref{prop: h}, we now present the main convergence result.

\begin{theorem}
	Consider the system \eqref{eq: L_mod_control} with $p(0) \in \mathbb{C}^n$, and choose $h$ as in Proposition \ref{prop: h}; then $p(t) \rightarrow \mathcal{S}$ and $\dot p(t) \rightarrow \mathcal{M}$ as characterized in Theorem \ref{thm: analytical_sol} as $t\to\infty$.
\end{theorem}
\begin{proof}
	Consider splitting $p(t)$ as in (\ref{eq: p_split}). Because of Proposition \ref{prop: h}, we know that if $h > h_l$, then $p_\mathcal{C}(t) \to 0$; thus, $p(t) \rightarrow p_\mathcal{S}(t) \in \mathcal{S}$ as $t \rightarrow \infty$. Consequently, we know that $\dot p(t) \rightarrow \mathcal{M}$, characterized in Theorem \ref{thm: analytical_sol}, as $p(t) \rightarrow \mathcal{S}$.
\end{proof}
\begin{coroll}[Input bound and saturation design] \label{cor: input}
Under the conditions of Theorem~\ref{thm: analytical_sol} with $h>h_l$, the input satisfies
\begin{equation*}
\|u(t)\|_2 \le \underbrace{h\,\|KLp_\mathcal{C}(t)\|_2}_{\text{transient}} + \underbrace{\kappa\,\|MB^\top p(t)\|_2}_{\text{motion}}.
\end{equation*}
The transient term decays exponentially by Proposition~\ref{prop: h} from its measured initial value $\|KLp(0)\|_2$, while the motion term is the commanded velocity $\kappa\|v(t)\|_2$, available in closed form from $\Delta_v$ and the analytical solution \eqref{eq: p_par_t}. For non-expanding motions ($\operatorname{Re}(l_\pm)\le 0$ in Proposition~\ref{prop: eigen_deg}, outside the resonant case C6), the motion term is uniformly bounded by $\nu := \sup_{t\ge 0}\|MB^\top p(t)\|_2$, computable from Theorem~\ref{thm: analytical_sol}; otherwise it grows and the bound holds over a finite horizon only. Consequently, taking $h$ just above $h_l$, any
\begin{equation*}
\kappa \le \frac{u_{\max}}{c_T\,\|Q\|_2\|MB^\top\|_2\,\|KLp(0)\|_2 + \nu},
\end{equation*}
with $c_T\ge 1$ bounding the transient overshoot of $\|KLp(t)\|_2$, guarantees $\|u(t)\|_2 \le u_{\max}$ for all $t\ge 0$. For higher-order platforms, $u(p)$ serves as the reference for the low-level controller enforcing actuator limits.
\end{coroll}

\section{Simulations} \label{sec: sims} 
To summarize our methodology, we implement Algorithm \ref{alg} for the numerical simulations presented in this section.

\begin{algorthm} \label{alg}
$ $

\begin{enumerate}
	\item Given a desired $p^*$ as in Lemma \ref{lem: L_eigen}, calculate the weights $\omega_{ij}$ of $L$ according to \cite[Proposition 2]{zhao2018affine}, that also requires the framework to be generically and universally rigid, so that they satisfy \eqref{eq: w_condition}.
	\item In case of considering an arbitrary framework for the design of the weights, calculate $K$ according to \cite[Theorem 3.3, Algorithm 3.1]{lin2013K}.
	\item Calculate the motion parameters $\mu_{ij}$ for $M_{v_x}$, $M_{v_y}$, $M_{v_{a_x}}$, $M_{v_{a_y}}$, $M_{v_{h_x}}$ and $M_{v_{h_y}}$ in \eqref{eq: M_split}, so that $M_{{\{v_x,v_y\}}} B^\top p^*$ corresponds to a horizontal/vertical translation of 1 \textit{distance unit/s}, $M_{\{v_{a_x}, v_{a_y}\}} B^\top p^*$ to a scaling of 1 \textit{current size/s}, and $M_{{\{v_{h_x}, v_{h_y}\}}} B^\top p^*$ to a shearing of 1 \textit{distance unit/s}. This process might vary depending on the chosen basis of collective motions.
	\item We finish the design of $M$ in \eqref{eq: M_split} by choosing the vector of affine coordinates $\Delta_v$ that determines the reference collective motion $v^*$ given by \eqref{eq: v_star}. Then, we choose $\kappa$ to modulate their global speed. Note that $\kappa = 1$ preserves the units from the step before.
	\item Finally, we design $h$ for the modified weights in \eqref{eq: wmod} according to Proposition \ref{prop: h}.
	\item Optionally, if the system starts at time $t_0$ with the configuration $p(t_0) \in \mathcal{S}$, we can calculate the $\alpha_{\{1,2,3\}}$ in \eqref{eq: p_par_t} that characterizes $p(t)$ for all $t \geq t_0$, i.e., there is no need of numerical integration since we know the exact analytic solution.
\end{enumerate}

\end{algorthm}

\begin{figure}[]
    \centering
    \includegraphics[trim={0cm 0cm 0cm 0cm}, clip, width=0.98\columnwidth]{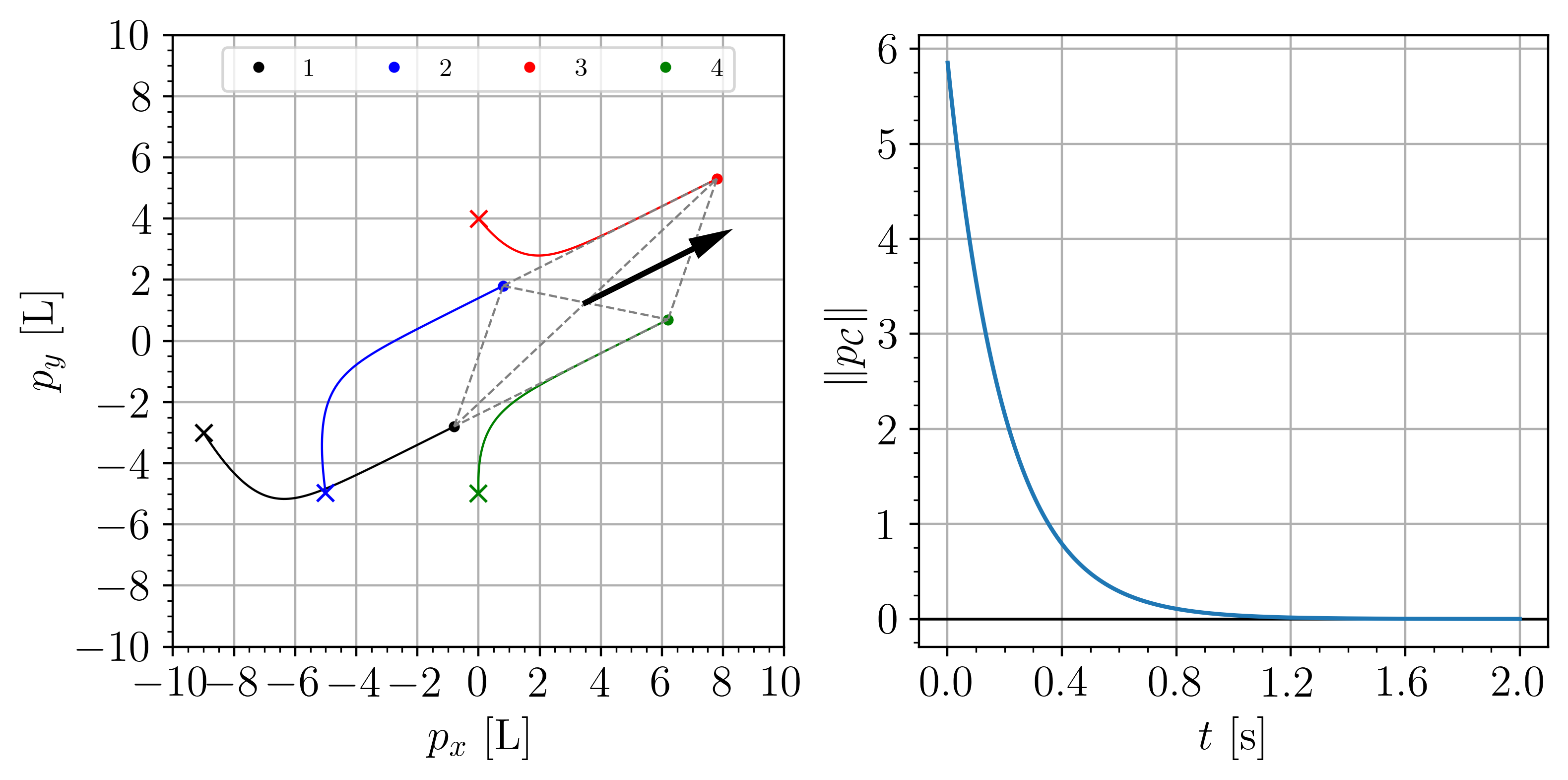}
	\caption{A formation of four agents converges to the desired shape $\mathcal{S}$, given by the reference shape $p^*_{\text{sim}}$, as their velocity converge to the desired collective motion $\mathcal{M}$ given by $v^* = \mathbf{1}_n$, denoted by the big black arrow. In the left plot, the solid lines represent the evolution of the agents' position, the crosses indicate their initial positions, and the grey dashed lines the interaction graph. The right plot represents the exponential convergence of $\|p_\mathcal{C}(t)\|$ to zero.}
    \label{fig: sim1}
\end{figure}

\begin{figure}[h!]
    \centering
    \includegraphics[trim={0cm 0cm 0cm 0cm}, clip, width=0.97\columnwidth]{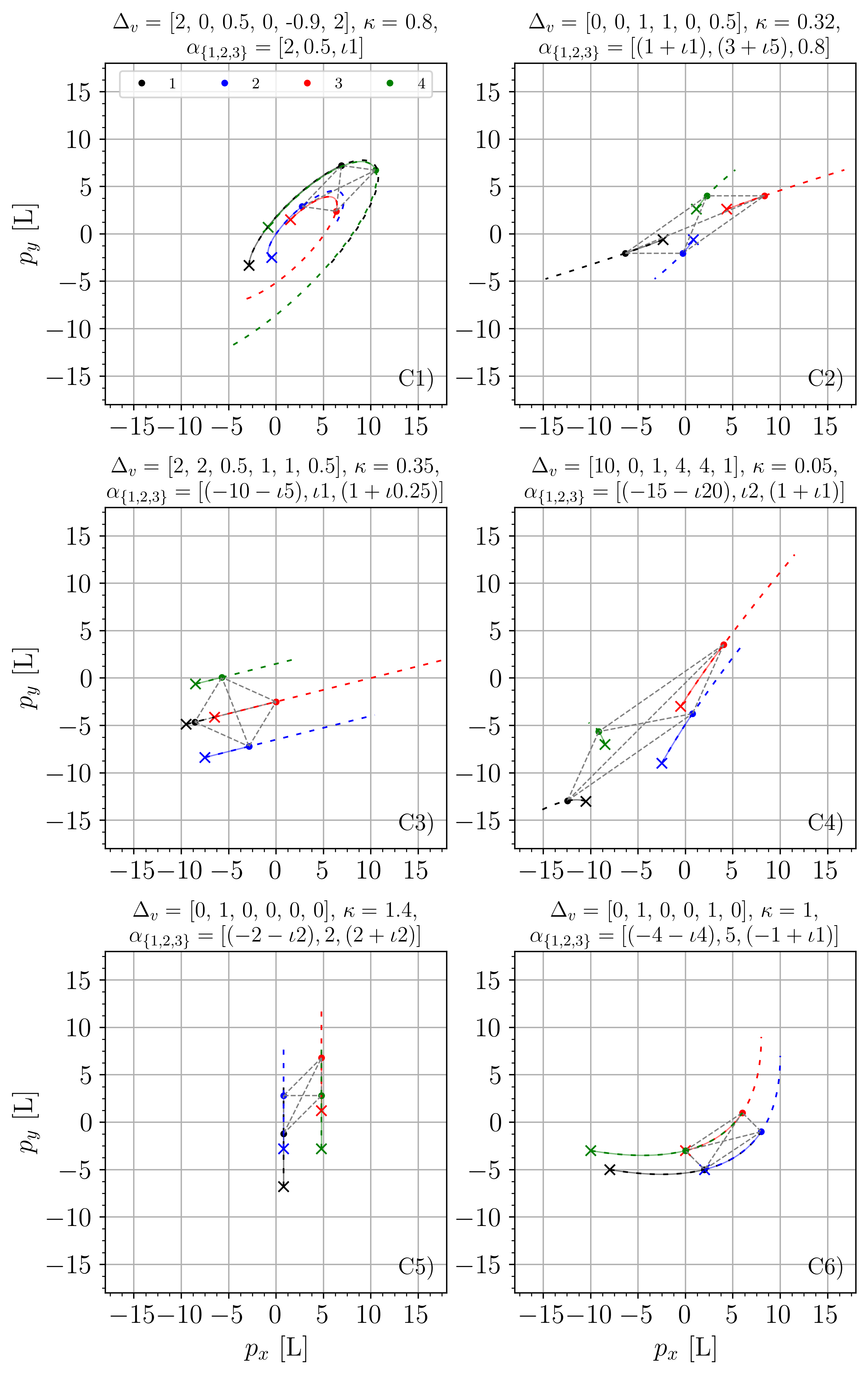}
	\caption{A formation of four agents preserves the desired shape $\mathcal{S}$, given by the reference shape $p^*_{\text{sim}}$, while moving with the desired collective motion $\mathcal{M}$ given by $v^* = T_{\Delta_v}(p^*_{\text{sim}})$. Each subplot corresponds to one of the six conditions in Theorem \ref{thm: analytical_sol}. Solid lines depict the Euler integration of the dynamics until $t_f = 3$ s, colored dashed lines represent the analytical solution up to $2 t_f$, crosses mark the initial positions, and grey dashed lines illustrate the interaction graph. We observe how the analytical solution coincides with the Euler numerical integration of the differential equation \eqref{eq: dyn} with $u = K\tilde L p$.}
    \label{fig: sim2}
\end{figure}

\begin{figure}[h!]
    \centering
    \includegraphics[trim={0cm 0cm 0cm 0cm}, clip, width=1\columnwidth]{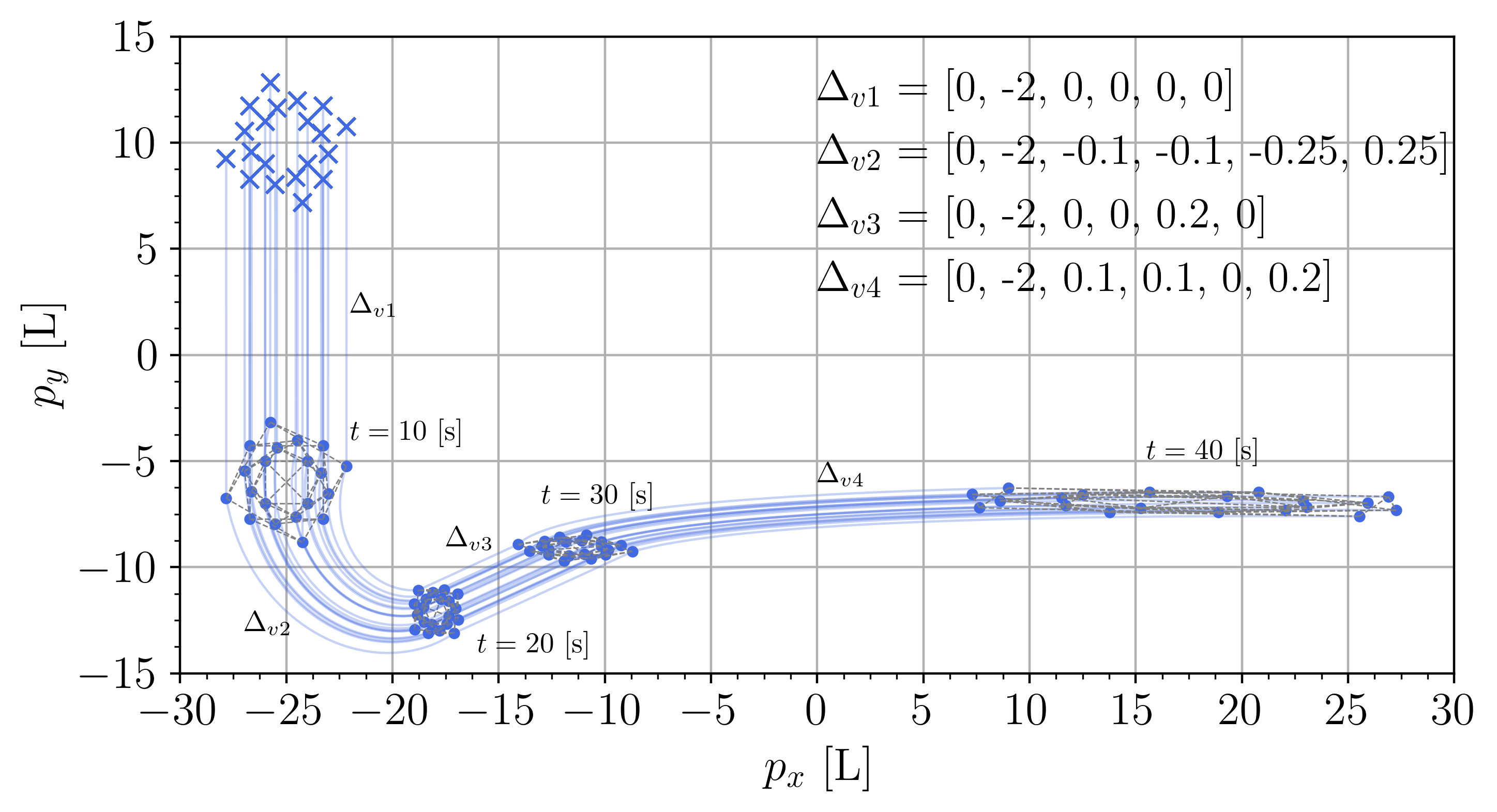}
    \caption{A formation of 20 agents preserves the desired shape $\mathcal{S}$, given by the reference shape $p^*_{\text{sim2}}$, as their velocity converge to the desired collective motion $\mathcal{M}$ given by $v^* = T_{\Delta_k}(p^*_{\text{sim2}})$, where $\Delta_k$ changes every 10 seconds. The solid lines represent the evolution of the agents' position, the crosses indicate their initial positions, given by $p(0) = p^*_{\text{sim2}} + (-25 + \iota 10)$, and the grey dashed lines the interaction graph.}
    \label{fig: sim3}
\end{figure}

In the following first two simulations, we consider a framework of four agents with a complete interaction graph. In both cases, the desired shape is given by the reference shape $p^*_{\text{sim}} = [-1-\iota,\; -1+\iota,\; 1+\iota,\; 1-\iota]^\top$, corresponding to the one illustrated in Figure \ref{fig: affine_v}. For this reference shape, we have numerically tested that $KL$, with $K = I_n$ and $
\scalemath{0.9}{
    L = \frac{1}{4}
    \left[\begin{smallmatrix}
        1 & -1 &  1 & -1 \\
       -1 &  1 & -1 &  1 \\
        1 & -1 &  1 & -1\\
       -1 &  1 & -1 &  1
    \end{smallmatrix}\right]
}
$, does not have eigenvalues with negative real component.
The six component matrices $M_{\{\cdot\}}B^\top$ for $p^*_{\text{sim}}$ and its interaction graph, along with the code for the three simulations below, are available in \cite{github}.

\subsection{Convergence to the desired shape and collective motion}

On the simulation depicted in Figure \ref{fig: sim1}, the team of four agents starts from an initial configuration that is not at the desired shape, i.e., $p(0) \notin \mathcal{S}$. The desired collective motion is given by $v^* = \mathbf{1}_n$, with $\kappa = 1$. To determine the lower bound of $h$ in Proposition \ref{prop: h} for this setup, we first compute $J_2 = 1$, which results in $Q = 1$. Next, we calculate $\|MB^\top\|_2 = 1$, and since $\kappa = 1$, \eqref{eq: hl} yields $h_l = 1$. Hence, we choose $h=5$ in order to satisfy the condition in Proposition \ref{prop: h}. This allows us to numerically demonstrate that $\|p_\mathcal{C}(t)\| \rightarrow 0$ exponentially fast as $t \rightarrow \infty$, and that $v(t) \to \mathcal{M}$ corresponding to the desired translational motion.

\subsection{Prediction of the stationary collective motions}

On the six simulations illustrated in Figure \ref{fig: sim2}, the team of agents starts from different initial configurations within $\mathcal{S}$. The desired reference collective motion $v^* = T_\Delta(p^*_{\text{sim}})$ and the parameter $\kappa$, together with the corresponding constants $\alpha_{\{1,2,3\}}$, are indicated in Figure \ref{fig: sim2} for each case. The Euler integration of the dynamics in \eqref{eq: dyn}, with $u = K\tilde L p$, match the analytical solution provided in Theorem \ref{thm: analytical_sol}. Note that in these simulations $hKLp(t) = 0$ for all $t\geq 0$ because $p(0) \in \mathcal{S}$ and $\mathcal{S}$ in invariant, so the design of $h$ is not a concern in this case.

\subsection{Maneuvering a large team of robots}

On the simulation illustrated in Figure \ref{fig: sim3}, we have a team of $n = 20$ agents with an interaction five-layer graph where each layer is a complete-graph square and consecutive layers are interconnected (full list in \cite{github}).
The desired shape is given by the reference shape 
\begin{align*}
    \scalemath{0.85}{
    p^*_{\text{sim2}} = [p^{*\top}_{\text{sim}},\; 
    p^{*\top}_{\text{sim}} \cdot (1.2 e^{\iota \frac{\pi}{3}}),\; 
    p^{*\top}_{\text{sim}} \cdot (1.2 e^{\iota \frac{\pi}{3}})^2, \; 
    \cdots, \; 
    p^{*\top}_{\text{sim}} \cdot (1.2 e^{\iota \frac{\pi}{3}})^4]^\top
    },
\end{align*}
for which we have numerically tested that $KL$, with $K = I_n$ and the computed $L$ \cite{github}, does not have eigenvalues with negative real component. The matrix $M$ was computed using a numerical least-squares approach to determine the motion parameters $\mu_{ij}$ that satisfies \eqref{eq: v_i_body} for each robot $i$ \cite{github}. The desired reference collective motion is given by $v^* = T_{\Delta_k}(p^*_{\text{sim2}})$, where $\Delta_k$ changes every 10 seconds as indicated in Figure \ref{fig: sim3}, and $\kappa = 0.8$. The design of $h$ is not a concern in this case because $p(0) \in \mathcal{S}$.


\section{Conclusions} \label{sec: conclusions}

This paper presents how to maneuver distributively and without leaders an \emph{affine formation}. In particular, we present a methodology to find the analytical solution of the single-integrators when the formation starts from a desired affine formation. To achieve this, we have encoded prior results from affine formation control in 2D \cite{lin2016necessary} into the complex domain, bridging in this way with the work in \cite{lin2016complexlap,hgdemarina_complex}. We then formally showed how to modify the original real Laplacian matrix to achieve a desired collective motion. Next, we have proved the convergence to both the desired shape and the desired collective motion, and calculate the eigenvalues and eigenvectors of the modified real Laplacian matrix that describe such a stationary collective motion. Finally, we have validated our results numerically by conducting three numerical simulations: the first one demonstrates the convergence to the desired affine formation simultaneously with the desired motion, the second simulation compares the analytical solution with an Euler integration of the dynamical system, and the third simulation showcases an application scenario involving a large number of robots.

Future research will focus on leveraging our analytical solution to study the impact of various imperfections, such as different perceptions among neighbors and communication failures. Additionally, exploring equivalence principles between hardware and software parameters can lead to adaptive control laws to mitigate potentially undesired stationary collective motions due to hardware imperfections.


\begin{appendix}

\subsection{Proof of Proposition \ref{prop: eigen_deg}} \label{ap: prop_eig}

The eigenvalue equation of $-K\tilde L$ for the eigenvectors $x_\lambda$ within $\mathcal{S}$ is given by \eqref{eq: eig_eq}, and can be reformulated using \eqref{eq: MB_phat_ker}, yielding the form shown in \eqref{eq: eig_constraints}. Thus, if $\lambda$ is an eigenvalue of $-K\tilde L$ associated with an eigenvector $x_\lambda \in \mathcal{S}$, the system of equations
\begin{align} \label{eq: prop1_system}
    \underbrace{
    \left[\begin{matrix}
        -\frac{\lambda}{\kappa} & v_x & v_y \\
        0 & v_{a_x} - \frac{\lambda}{\kappa} & v_{h_y} \\
        0 & v_{h_x} & v_{a_y} - \frac{\lambda}{\kappa} \\
    \end{matrix}\right]
    }_{(A_{v^*} - I_3\frac{\lambda}{\kappa})}
    \underbrace{
    \left[\begin{matrix}
        c_1 \\ c_2 \\ c_3
    \end{matrix}\right]}_{[x_\lambda]_{\mathcal{B}_\mathcal{S}}} = 0,
\end{align}
must hold. Therefore, calculating these eigenvalues of $-K\tilde L$ is equivalent to finding the roots of $\operatorname{det}(A_{v^*} - I_3\frac{\lambda}{\kappa}) = 0$, which yields the eigenvalues $\{0, l_+, l_-\}$.

On the other hand, the eigenspace $\mathcal{U}_\lambda$ of each (potentially degenerated) eigenvalue $\lambda$ calculated previously is given by the set of vectors $x_\lambda \in \mathcal{S}$ which satisfy \eqref{eq: prop1_system}, i.e., $\mathcal{U}_\lambda := \{x_\lambda \in \mathcal{S} \; | \; (A_{v^*} - I_3 \frac{\lambda}{\kappa}) x_\lambda = 0\}$. We further denote by $\mathcal{U}_\lambda^k$ the eigenspace associated with the generalized eigenvector $x_\lambda^k$ of rank $k \in \mathbb{N}_+$, forming part of a Jordan chain. This eigenspace consists of the vectors $x_\lambda^k \in \mathcal{S}$ that satisfy the condition $(-K\tilde L - I_n \lambda) x_\lambda^k = x_\lambda^{k-1}$, i.e.,
\begin{align} \label{eq: prop1_system_gen}
    \left[\begin{matrix}
        -\frac{\lambda}{\kappa} & v_x & v_y \\
        0 & v_{a_x} - \frac{\lambda}{\kappa} & v_{h_y} \\
        0 & v_{h_x} & v_{a_y} - \frac{\lambda}{\kappa} \\
    \end{matrix}\right]
    \underbrace{
    \left[\begin{matrix}
        c_1^k \\ c_2^k \\ c_3^k
    \end{matrix}\right]}_{[x_\lambda^k]_{\mathcal{B}_\mathcal{S}}} = 
    \underbrace{
        \left[\begin{matrix}
            c_1^{k-1} \\ c_2^{k-1} \\ c_3^{k-1}
        \end{matrix}\right]}_{[x_\lambda^{k-1}]_{\mathcal{B}_\mathcal{S}}} \frac{1}{\kappa}.
\end{align}
Note that a generalized eigenvector of rank $k=1$ is an ordinary eigenvector, in which case we recover the system of equations given by \eqref{eq: prop1_system}. The eigenspaces $\mathcal{U}_\lambda^k$ are dependent on the design of $v^*$ in \eqref{eq: v_star}, as $v^*$ strictly dictates the structure of $A_{v^*}$. Consequently, the following extensive analysis focuses on identifying all non-trivial $\mathcal{U}_\lambda^k$ for the cases presented in the statement:

\begin{enumerate}
    \item [C1)] Solving \eqref{eq: prop1_system} for $\lambda = 0$ yields $c_1$ as an arbitrary parameter, $c_2 = -\frac{v_{h_y}}{v_{a_x}}c_3$, and for $c_3$, the following two equations must be satisfied
    \begin{align} \label{eq: c3_system}
        \scalemath{0.9}{
        \begin{cases}
            \begin{array}{l}
                c_3 [v_{h_y}v_x - v_{a_x} v_y] = 0 \\
                c_3 [v_{a_y} v_{a_x} - v_{h_x}v_{h_y}] = 0.
            \end{array}
        \end{cases}
        }
    \end{align}
    Since in this case $l_\pm \neq 0$, we have that $v_{a_y} v_{a_x} \neq v_{h_x} v_{h_y}$, thus the only solution for \eqref{eq: c3_system} is $c_3 = 0$. Consequently, $\mathcal{U}_0 = \{c_1 \mathbf{1}_n \; | \; c_1 \in \mathbb{C}\}$ and $\mathbf{1}_n \in \mathcal{U}_0$ is a valid eigenvector for $\lambda = 0$. 
    
    On the other hand, solving \eqref{eq: prop1_system} for $\lambda = l_\pm$ yields $c_1 = \frac{\kappa}{l_\pm}(v_x c_2 + v_y c_3)$, $c_2 = \frac{v_{h_y}}{l_\pm/\kappa - v_{a_x}}$, and $c_3$ as an arbitrary parameter, i.e., 
    \begin{align*}
        \scalemath{0.9}{
        \begin{array}{l}
            \mathcal{U}_{l_\pm} = \left\{ \frac{\kappa}{l_\pm} c_3 \left[ \frac{v_{h_y} v_x}{l_\pm/\kappa - v_{a_x}} + v_y \right] \mathbf{1}_n + \right. \\
            \qquad \qquad \qquad \left. c_3 \left[\frac{v_{h_y}}{l_\pm/\kappa - v_{a_x}} \operatorname{Re}(p^*) + \operatorname{Im}(p^*)\right] \; | \; c_3 \in \mathbb{C} \right\}.
        \end{array}
        }
    \end{align*}
    Therefore, considering $c_3 = (l_\pm/\kappa - v_{a_x})$ yields $x_{l_\pm} \in \mathcal{U}_{l_\pm}$ in \eqref{eq: c1_lpm}, which is a valid pair of eigenvectors for the eigenvalues $l_\pm$. 

    In the case when $l_+ = l_- = l$, $v_{a_x} = v_{a_y}$ and $v_{h_x} = v_{h_y} = 0$, we have that $l = \frac{v_{a_x} + v_{a_y}}{2}$. Solving \eqref{eq: prop1_system} $c_1$ remains as before, but this time both $c_2$ and $c_3$ are arbitrary parameters, i.e., $\mathcal{U}_{l_\pm} = \{ \frac{2}{v_{a_x} + v_{a_y}} (v_x c_2 + v_y c_3)\mathbf{1}_n + c_2 \operatorname{Re}(p^*) + c_3\operatorname{Im}(p^*) \; | \; c_2,c_3 \in \mathbb{C}\}$. Therefore, considering $c_2 = (v_{a_x} + v_{a_y})/2, c_3 = 0$ and $c_2 = 0, c_3 = (v_{a_x} + v_{a_y})/2$ yields $x_{l_\pm} \in \mathcal{U}_{l_\pm}$ in \eqref{eq: c1_lp} and \eqref{eq: c1_lm}, which is a valid pair of linearly independent eigenvectors for the degenerated eigenvalue $l = l_+ = l_-$.

    \item [C2)] The analysis for $\lambda = 0$ remains as in C1). When $l_+ = l_- = l$, $v_{a_x} = v_{a_y}$, and either $v_{h_x} = 0$ or $v_{h_y} = 0$, solving \eqref{eq: prop1_system} yields $\mathcal{U}_{l}^1 = \{ f(\alpha) \; | \; \alpha \in \mathbb{C} \}$, where
    \begin{align*}
        \scalemath{0.9}{
        f(\alpha) =
        \begin{cases}
            \begin{array}{rl}
                \frac{2 v_x}{v_{a_x} + v_{a_y}} \alpha \mathbf{1}_n + \alpha \operatorname{Re}(p^*) & \mbox{if } v_{h_x} = 0 \\
                \frac{2 v_y}{v_{a_x} + v_{a_y}} \alpha \mathbf{1}_n + \alpha \operatorname{Im}(p^*) & \mbox{if } v_{h_y} = 0.
            \end{array}
        \end{cases}
        }
    \end{align*}
    Considering $\alpha = \kappa\frac{v_{a_x} + v_{a_y}}{2}$ leads us to $x_{l}^1 \in \mathcal{U}_{l}^1$ in \eqref{eq: c2_xl_1}, which is a valid generalized eigenvector of rank $k = 1$. On the other hand, the dimension of $\mathcal{U}_{l}^1$ is one order less than the algebraic multiplicity of $l$, so we have a non-trivial generalized eigenvector of rank $k=2$. Solving \eqref{eq: prop1_system_gen} yields
    $\mathcal{U}_{l}^2 = \{ g(\beta) \; | \; \beta \in \mathbb{C}\}$, where
    \begin{align*}
        \scalemath{0.7}{
        g(\beta) =
        \begin{cases}
            \begin{array}{rl}
                \begin{aligned}
                    &\frac{2}{\kappa (v_{a_x} + v_{a_y})} \left[ \beta v_x + \alpha \left(\frac{v_y}{v_{h_y}} - \frac{2 v_x}{v_{a_x} + v_{a_y}} \right)\right] \mathbf{1}_n \\
                    & \qquad\qquad + \frac{\beta}{\kappa} \operatorname{Re}(p^*) + \frac{\alpha}{\kappa v_{h_y}}\operatorname{Im}(p^*)
                \end{aligned}
                & \mbox{if } v_{h_x} = 0 \\
                \begin{aligned}
                    &\frac{2}{\kappa (v_{a_x} + v_{a_y})} \left[ \beta v_y + \alpha \left(\frac{v_x}{v_{h_x}} - \frac{2 v_y}{v_{a_x} + v_{a_y}} \right)\right] \mathbf{1}_n \\
                    & \qquad\qquad + \frac{\alpha}{\kappa v_{h_x}}\operatorname{Re}(p^*) + \frac{\beta}{\kappa}\operatorname{Im}(p^*)
                \end{aligned}
                & \mbox{if } v_{h_y} = 0.
            \end{array}
        \end{cases}
        }
    \end{align*}
    Therefore, considering the previous $\alpha$ for $\mathcal{U}_{l}^1$ and $\beta = \kappa$ yields $x_l^2$ in \eqref{eq: c2_xl_2}, which is a valid generalized eigenvector of rank $k = 2$. Indeed, $\{x_l^2, x_l^1\}$ is a Jordan chain.
   
    \item [C3)] The analysis for the nonzero $l_\pm$ remains as in C1). Regarding the eigenvalue $\lambda = 0$ with algebraic multiplicity 2, the solution of \eqref{eq: prop1_system} is the same as for the zero eigenvalue in C1). However, this time $v_{a_y} v_{a_x} = v_{h_x}v_{h_y}$ because one of the eigenvalues $l_\pm$ is equal to zero, and $v_{h_y} v_x = v_{a_x} v_y$, so $c_3$ is also an arbitrary parameter. Therefore, $\mathcal{U}_0 = \{\alpha \mathbf{1}_n + \beta \left[ \frac{- v_y}{v_x} \operatorname{Re}(p^*) + \operatorname{Im}(p^*)\right] \; | \; \alpha,\beta \in \mathbb{C}\}$. Considering $\alpha = 1, \beta = 0$ and $\alpha = 0, \beta = -v_x$, yields $\mathbf{1}_n$ and $x_0$ in \eqref{eq: c3_x0}, respectively. These two eigenvector associated with $\lambda = 0$ are linearly independent, so they span $\mathcal{U}_0$. Additionally, note that $v_{h_y} v_x = v_{a_x} v_y$ implied that $v_{h_x} v_y = v_{a_y} v_x$, as the condition $v_{a_y} v_{a_x} = v_{h_x} v_{h_y}$ ensures consistency between the two expressions.

    \item [C4)] The analysis remains as in C3). However, since in this case $v_{h_y} v_x \neq v_{a_x} v_y$, solving \eqref{eq: prop1_system} for $\lambda = 0$ yields $c_3 = 0$, i.e., $\mathcal{U}_0^1 = \{c_1 \mathbf{1}_n \; | \; c_1 \in \mathbb{C}\}$. Hence, setting $\alpha = \kappa [v_x v_{h_y} - v_{a_x} v_y]$ yields $\kappa [v_x v_{h_y} - v_{a_x} v_y] \mathbf{1}_n$, which is a valid generalized eigenvector of rank 1. Solving \eqref{eq: prop1_system_gen} for $k=2$ yields $\mathcal{U}_0^2 = \{ \beta \mathbf{1}_n + \frac{\alpha}{\kappa [v_{a_x} v_y - v_x v_{h_y}]}(-v_{h_y} \operatorname{Re}(p^*) + v_{a_x} \operatorname{Im}(p^*)) \; | \; \beta \in \mathbb{C}\}$. Considering the previous $\alpha$ and choosing $\beta = 0$ yields $x_0^2$ in \eqref{eq: c4_x0_2}, which is a valid generalized eigenvector of rank 2, thus $\{x_0^2, x_0^1\}$ is a Jordan chain.

    \item [C5)] When either $v_{h_x},v_x = 0$ or $v_{h_y},v_y = 0$, solving \eqref{eq: prop1_system} for $\lambda = 0$ yields $\mathcal{U}_0 = \{\alpha \mathbf{1}_n + \beta \operatorname{Re}(p^*)\; | \; \alpha,\beta \in \mathbb{C}\}$ if $v_{h_x}, v_x = 0$, and $\mathcal{U}_0 = \{\alpha \mathbf{1}_n + \beta \operatorname{Im}(p^*)\; | \; \alpha,\beta \in \mathbb{C}\}$ if $v_{h_y}, v_y = 0$. Then, solving \eqref{eq: prop1_system_gen} with $x_0^1 \in \mathcal{U}_0$ leads to $\mathcal{U}_0^2 = \{\gamma \mathbf{1}_n + f(\rho)\; | \; \gamma,\rho \in \mathbb{C}\}$, where
    $$
    \scalemath{0.9}{
    f(\rho) =
    \begin{cases}
        \begin{array}{rl}
            \rho \operatorname{Re}(p^*) + \operatorname{Im}(p^*)
            & \mbox{if } v_{h_x},v_x = 0 \\
            \operatorname{Re}(p^*) + \rho \operatorname{Im}(p^*)
            & \mbox{if } v_{h_y},v_y = 0 \\
        \end{array}
    \end{cases}
    }
    $$
    and requires $\alpha = \kappa v_y$ and $\beta = \kappa v_{h_y}$ if $v_{h_x}, v_x = 0$, or  $\alpha = \kappa v_x$ and $\beta = \kappa v_{h_x}$ if $v_{h_y}, v_y = 0$. These conditions yield $x_0^1$ as given in \eqref{eq: c5_x01}. The vector in $\mathcal{U}_0$ orthogonal to $x_0^1$ produces $y_0$ as given in \eqref{eq: c5_y0}, a valid eigenvector for the non-degenerated eigenvalue. Finally, setting $\gamma = \rho = 0$ gives $x_0^2$ in \eqref{eq: c5_x02}.
    
    When $v_{h_x}, v_{h_y} = 0$ and either $v_x \neq 0$ or $v_y \neq 0$, a potential solution to \eqref{eq: prop1_system} is given by $\mathcal{U}_0 = \{\alpha \mathbf{1}_n + \beta [\operatorname{Re}(p^*) - \frac{v_x}{v_y}\operatorname{Im}(p^*)] \; | \; \alpha, \beta \in \mathbb{C} \}$. On the other hand, solving \eqref{eq: prop1_system_gen} with $x_0^1 \in \mathcal{U}_0$ results in $\mathcal{U}_0^2 = \{\gamma \mathbf{1}_n + \rho\left[ \left( \frac{\alpha}{\kappa v_x} - \frac{v_y}{v_x} \right) \operatorname{Re}(p^*) + \operatorname{Im}(p^*)\right] \; | \; \gamma, \rho \in \mathbb{C}\}$, requiring $\beta = 0$, i.e., $\mathcal{U}_0^1 = \{ \alpha \mathbf{1}_n \; | \; \alpha \in \mathbb{C}\}$. Consequently, the eigenspace associated with the non-degenerated eigenvalue is $\mathcal{U}_0^\prime = \mathcal{U}_0 - \mathcal{U}_0^1 = \{ \beta [\operatorname{Re}(p^*) - \frac{v_x}{v_y}\operatorname{Im}(p^*)] \; | \; \beta \in \mathbb{C} \}$. Setting $\alpha = \kappa (v_x + \iota v_y)$, $\gamma = 0$ and $\rho = v_x$ yields the Jordan chain $\{x_0^2, x_0^1\}$, with $x_0^2 \in \mathcal{U}_0^2$ and $x_0^1 \in \mathcal{U}_0^1$ as in \eqref{eq: c5_x01} and \eqref{eq: c5_x02}. Additionally, choosing $\beta = v_y$ provides $y_0 \in \mathcal{U}_0^\prime$ in \eqref{eq: c5_y0}.

    \item [C6)] Solving \eqref{eq: prop1_system} for the eigenvalue $\lambda = 0$ with algebraic multiplicity 3 yields $\mathcal{U}_0^1 = \{ \alpha \mathbf{1}_n \; | \; \alpha \in \mathbb{C}\}$. Then, solving \eqref{eq: prop1_system_gen} for $k = 2$ yields $\mathcal{U}_0^2 = \{ f(\beta) \; | \; \beta \in \mathbb{C}\}$, and solving for $k = 3$ yields $\mathcal{U}_0^3 = \{ g(\gamma) \; | \; \gamma \in \mathbb{C}\}$, where
    \begin{align*}
        \scalemath{0.9}{
        f(\beta) =
        \begin{cases}
            \begin{array}{rl}
                \beta \mathbf{1}_n + \frac{\alpha}{\kappa v_y} \operatorname{Im}(p^*)
                & \mbox{if } v_{h_x},v_y \neq 0 \\
                \beta \mathbf{1}_n + \frac{\alpha}{\kappa v_x} \operatorname{Re}(p^*)
                & \mbox{if } v_{h_y},v_x \neq 0,
            \end{array}
        \end{cases}
        }
    \end{align*}
    \begin{align*}
        \scalemath{0.85}{
        g(\gamma) =
        \begin{cases}
            \begin{array}{rl}
                \begin{aligned}
                    &\gamma \mathbf{1}_n + \frac{\alpha}{\kappa^2 v_y v_{h_x}} \operatorname{Re}(p^*)\\
                    &\qquad + \left( \frac{\beta}{v_y} - \frac{v_x \alpha}{\kappa^2 v_y^2 v_{h_x}}\right) \operatorname{Im}(p^*)
                \end{aligned}
                & \mbox{if } v_{h_x},v_y \neq 0 \\
                \begin{aligned}
                    &\gamma \mathbf{1}_n + \frac{\alpha}{\kappa^2 v_x v_{h_y}} \operatorname{Im}(p^*) \\
                    &\qquad + \left( \frac{\beta}{v_x} - \frac{v_y \alpha}{\kappa^2 v_x^2 v_{h_y}}\right) \operatorname{Re}(p^*)
                \end{aligned}
                & \mbox{if } v_{h_y},v_x \neq 0.
            \end{array}
        \end{cases}
        }
    \end{align*}
    Considering $\alpha = \kappa^2 v_y^2 v_{h_x}$ when $v_{h_x}, v_y \neq 0$, or $\alpha = \kappa^2 v_x^2 v_{h_y}$ when $v_{h_y}, v_x \neq 0$, along with $\beta = \kappa v_x v_y$ and $\gamma = 0$, yields the Jordan chain $\{x_0^3, x_0^2, x_0^1\}$, with $x_0^3, x_0^2$ and $x_0^1$ as in \eqref{eq: c6_x01}-\eqref{eq: c6_x03}.

\end{enumerate}

\end{appendix}


\bibliographystyle{IEEEtran}
\bibliography{biblio}


\begin{IEEEbiography}[{\includegraphics[width=1in, height=1.25in, clip,keepaspectratio, trim={0cm 0cm 0cm 0cm}]{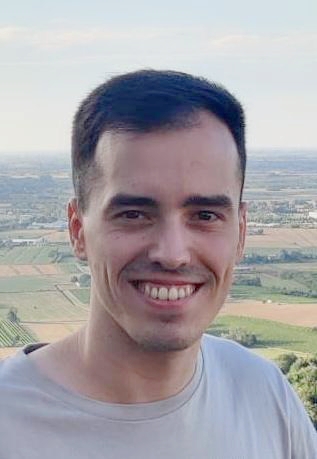}}]{Jesus Bautista Villar} (Student IEEE) received his B.S. degree in Physics from the Complutense University of Madrid, Spain, in 2022 and his M.S. degree in Data Science and Computer Engineering from the University of Granada, Spain, in 2023. He is currently pursuing a Ph.D. in Information and Communication Technologies, specializing in Cognitive Systems and Robotics, at the University of Granada, Spain.
\end{IEEEbiography}

\begin{IEEEbiography}[{\includegraphics[width=1in, height=1.25in, clip,keepaspectratio, trim={0cm 3cm 0cm 0cm}]{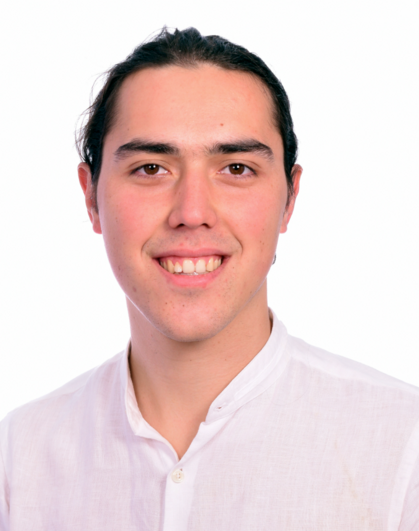}}]{Enric Morella Violeta} (Student IEEE) completed his double B.S. degree in Physics and Mathematics from the Complutense University of Madrid, Spain, in 2024 and and his M.S. degree in Data Science and Computer Engineering from the University of Granada, Spain, in 2025.
\end{IEEEbiography}

\begin{IEEEbiography}[{\includegraphics[width=1in,height=1.25in,clip,keepaspectratio]{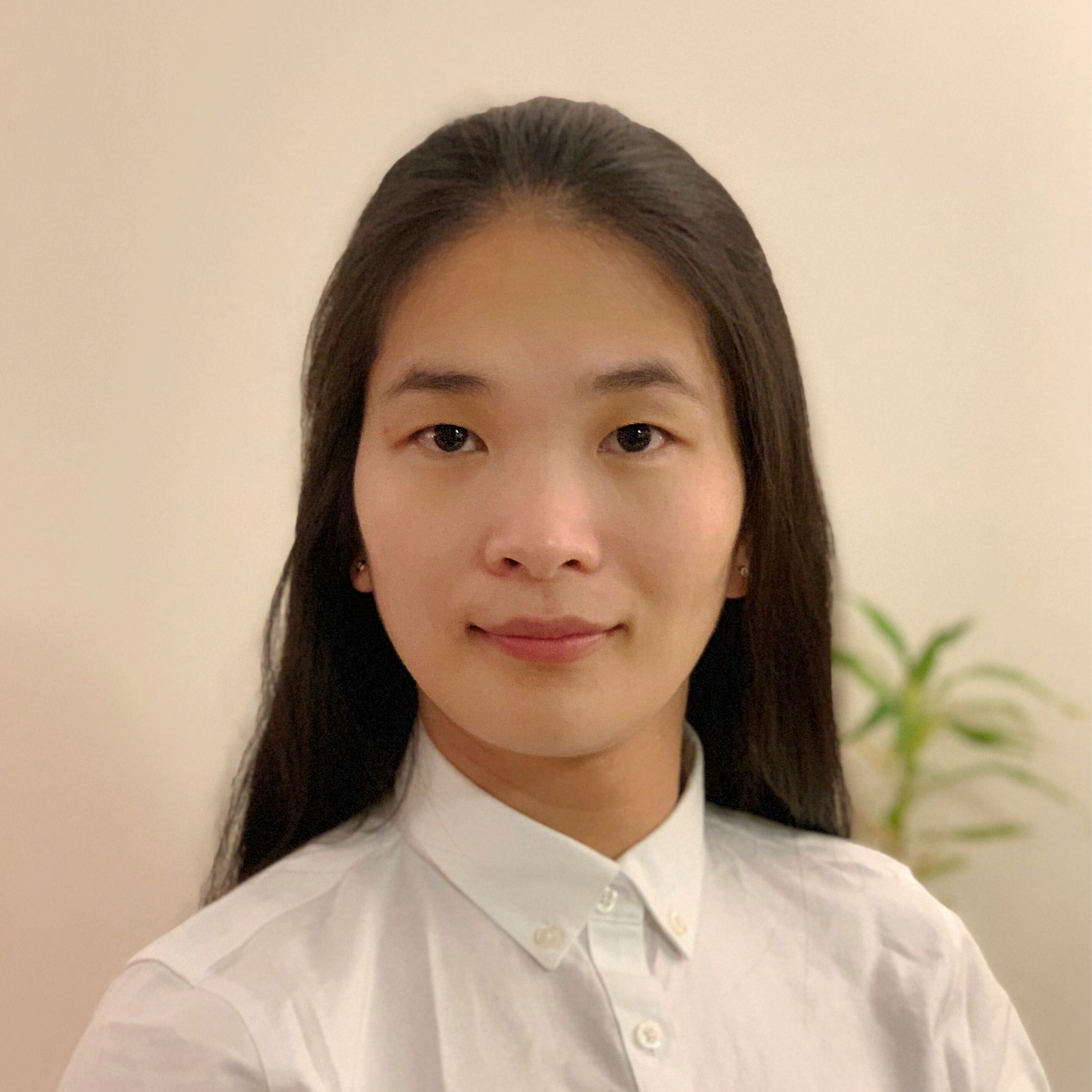}}]{Lili Wang} (Member IEEE) received the B.E. and M.S. degrees in electrical engineering from Zhejiang University, Zhejiang, China, in 2011 and 2014, respectively, and the Ph.D. degree in electrical engineering from Yale University, New Haven, CT, USA, in 2020. She is currently an Associate Professor at School of Automation and Intelligent Manufacturing, Southern University of Science and Technology, Shenzhen, China. Her research interests include the topic of cooperative multiagent systems, distributed computation and estimation, distributed control, and social networks.
\end{IEEEbiography}

\begin{IEEEbiography}[{\includegraphics[width=1in,height=1.25in,clip,keepaspectratio]{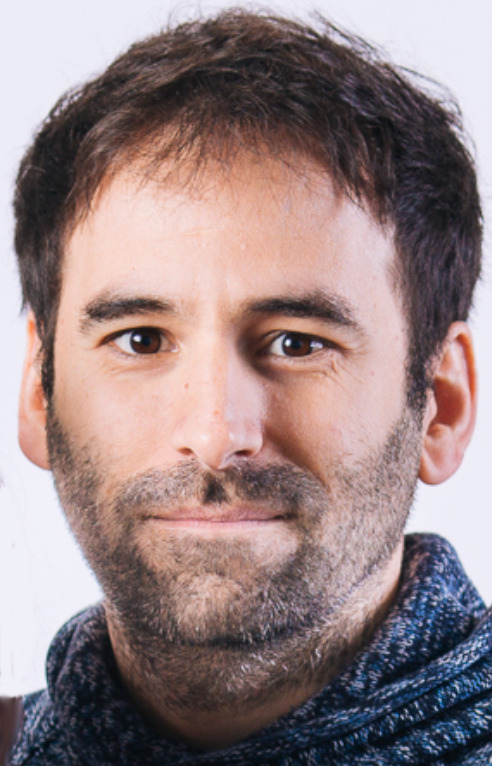}}]{Hector Garcia de Marina} (Member IEEE) received the Ph.D. degree in systems and control from the University of Groningen, The Netherlands, in 2016. He is currently a Ramón y Cajal Researcher with the Department of Computer Engineering, Automation and Robotics, and with the Institute of Mathematics (IMAG) at the Universidad de Granada, Spain. He is the recipient of an ERC Starting Grant and an Associate Editor for IEEE Transactions on Robotics. His current research interests include multiagent systems and the design of guidance navigation and control systems.
\end{IEEEbiography}

\end{document}